\documentclass[letterpaper]{article} % DO NOT CHANGE THIS
\usepackage{aaai24}
\usepackage{times}  % DO NOT CHANGE THIS
\usepackage{helvet}  % DO NOT CHANGE THIS
\usepackage{courier}  % DO NOT CHANGE THIS
\usepackage[hyphens]{url}  % DO NOT CHANGE THIS
\usepackage{graphicx} % DO NOT CHANGE THIS
\usepackage{natbib}  % DO NOT CHANGE THIS AND DO NOT ADD ANY OPTIONS TO IT
\usepackage{caption} % DO NOT CHANGE THIS AND DO NOT ADD ANY OPTIONS TO IT
\usepackage{algorithm}
\usepackage{algorithmic}

\usepackage{multirow}
\usepackage{array}

\usepackage[toc]{appendix}

\usepackage{newfloat}
\usepackage{listings}
\DeclareCaptionStyle{ruled}{labelfont=normalfont,labelsep=colon,strut=off} % DO NOT CHANGE THIS
\floatstyle{ruled}
\newfloat{listing}{tb}{lst}{}
\floatname{listing}{Listing}
\usepackage[utf8]{inputenc} % allow utf-8 input
\usepackage[T1]{fontenc}    % use 8-bit T1 fonts
\usepackage{url}            % simple URL typesetting
\usepackage{booktabs}       % professional-quality tables
\usepackage{amsfonts}       % blackboard math symbols
\usepackage{nicefrac}       % compact symbols for 1/2, etc.
\usepackage{microtype}      % microtypography
\usepackage{xcolor}         % colors

\usepackage{amsthm,amsmath,amssymb}

\usepackage{amsmath}
\usepackage{xcolor}

\newtheorem{definition}{Definition}[section]
\newtheorem{theorem}{Theorem}[section]

\def\~#1{\mathbb{#1}}
\def\*#1{\mathbf{#1}}
\def\@#1{\mathcal{#1}}

\usepackage{caption}
\usepackage{subcaption}
\usepackage{graphicx}

\usepackage{comment}

\title{Simple Weak Coresets for Non-Decomposable Classification Measures}
\author{
   Jayesh Malaviya\textsuperscript{\rm 1},
    Anirban Dasgupta\textsuperscript{\rm 1},
    Rachit Chhaya\textsuperscript{\rm 2}
}
\affiliations{
   \textsuperscript{\rm 1}Indian Institute of Technology, Gandhinagar\\
    \textsuperscript{\rm 2}DA-IICT, Gandhinagar\\
     malaviya$\_$jayesh@iitgn.ac.in, 
     anirbandg@iitgn.ac.in, 
     rachit$\_$chhaya@daiict.ac.in
}

\usepackage{bibentry}
\begin{document}

\maketitle

\begin{abstract}
While coresets have been growing in terms of their application, barring few exceptions, they have mostly been limited to unsupervised settings. We consider supervised classification problems, and non-decomposable evaluation measures in such settings. We show that stratified uniform sampling based coresets have excellent empirical performance that are backed by theoretical guarantees too. We focus on the F1 score and Matthews Correlation Coefficient, two widely used non-decomposable objective functions that are nontrivial to optimize for, and show that uniform coresets attain a lower bound for coreset size, and have good empirical performance, comparable with ``smarter'' coreset construction strategies. 
\end{abstract}

\section{Introduction}

% \begin{itemize}
% \item coresets--- broad introduction.
% \item techniques for decomposable loss functions

% \item importance of non-decomposable loss functions.

% \item inapplicability of sensitivity based measures for coresets. 

% \item current complexity of some of the non-decomposable algorithms. 
% \end{itemize}

 In typical classification tasks, multiple objective functions are often at stake--- one is the (surrogate) classification objective being optimized, and the other is the real measure of performance. This ``real measure'' is non-differentiable and, more importantly, often a non-decomposable function over the set of input points, i.e., it cannot be written as a sum over the input points. Popular examples include $F_1$ score, Matthews Correlation Coefficient (MCC), area under the curve (AUC-ROC) measures, and various variants of combinations of precision and recall measures e.g. H-mean, G-mean etc. In spite of being non-differentiable and non-decomposable, such classification measures are attractive since they help portray the tradeoffs between precision and recall, help enable the handling of mild to severe label imbalance, etc.

Such non-additive measures are generally optimized empirically by choosing a surrogate decomposable measure, then optimizing the decision boundary threshold using grid search. In recent years, a host of algorithms have been developed for optimizing the non-decomposable measures directly~\cite{joachims2005support, nan2012optimizing, narasimhan2015optimizing, Google_scalable_nondecomposable}. However, the main drawback remains that the algorithms developed are often not very efficient in practice and, thereby, remain harder to scale over large datasets. 

For decomposable loss functions, such as for regression or matrix factorization, one way in which this complexity of training has been circumvented is by constructing coresets--- summaries of the dataset that enable an optimizer to  optimize over the coreset only and give both a theoretically guaranteed as well as empirically satisfying performance for the full data. The major tool for coreset creation remains importance sampling, and that relies crucially on the decomposability of the loss function in order to deduce the importance weights. 
Indeed, in traditional classification coreset literature, researchers have explored various coreset building mechanisms for logistic loss \cite{on_coreset_logistic}, SVM \cite{on_coreset_svm}, etc. However, again the main catch is their evaluation and bounds focus on the decomposable surrogates such as the logistic loss function, hinge loss function, etc. \cite{classification_coreset}. For non-decomposable performance measures e.g.  F1 score, MCC score, AUC etc. there does not exist any such coreset creation mechanisms.

% However, in the classification setting, it might be more important to instead preserve evaluation measures such as F1-score and MCC(Matthew’s correlation coefficient). Unlike cost functions on which coreset techniques are typically applied, such evaluation measures may not be decomposable. 

Our paper is the first step along this direction. In this work, we show that for F1-score and MCC, applying stratified uniform sampling
is enough to obtain a (weak) coreset that preserves the value of these measures up to a small additive error for an interesting set of queries that contains the optimal. We also show a lower bound for strong coresets for the F1 and the MCC scores, implying that we cannot do much better than uniform sampling. We provide experimental evidence of our results on real datasets and for various classifiers and sampling techniques. Even though we formally prove our bounds and theory for F1-score and MCC, it applies to most contingency table-based measures.

Following are our main contributions:
\begin{enumerate}
    \item We provide lower bounds against construction of strong coresets for both MCC and F1 score measures.
    \item We show that sampling uniformly from each class in a stratified manner gives a weak coreset for the F1 score and a weak coreset with a small additive error for MCC. Here weak signifies that our coreset works for a set of a large number of 'important' classifiers, including the optimal.
    \item We provide empirical results for a number of different classifiers and real data sets comparing uniform sampling with other well-known sophisticated coreset construction strategies.
    
\end{enumerate}
The rest of the paper is organized as follows: Sections 2 and 3 give the necessary background and related work. Section 4 provides lower bounds(negative results) for coresets for both MCC and F1 scores. Sections 5 and 6 give the analysis of our weak coreset construction guarantees for F1 score and MCC. We discuss the experiments and comparison of uniform sampling on real data sets with a variety of sampling algorithms in section 7 and conclude in section 8.

%  Therefore, we attempt to design a first coreset-based subset selection method for loss functions like f1-score, MCC, etc. However, our finding found that we can not do better than stratified uniform sampling using sophisticated techniques. We prove that uniform sampling is the best we can do in theory. Our experiments also confirm that stratified uniform sampling matches all the sophisticated methods and often beats. 
% \textcolor{blue}{-- Coreset intro --------}

\section{Background}
A coreset can be considered to be a weighted subset of data from the original dataset or some different small-sized representation of the original dataset. In designing a coreset, the main challenge is to build one with provable tradeoffs between the approximation error and the coreset size.

We  interchangeably refer to classifiers as "query" and denote them by $q$. $Q$ denotes the set of all possible classifiers.

\begin{definition} (Coreset) \cite{agarwal2005geometric}  Given a weighted dataset $X$, let $x \in X$ and $\mu_{X}(x)$ be its corresponding nonnegative weight. Let $Q$ be a set of solutions known as query space and $q \in Q$ be a query. For each $q\in Q$, let $f_q: X \rightarrow \mathbb{R}_{\ge 0}$ be a non-negative function. Define  $cost(X,q) = \sum_{x \in X} \mu_{X}(x) f_{q}(x).$ For $ \epsilon > 0$, a weighted-set $(C, w)$ is an $ \epsilon$-coreset of $X$ for the cost function $\{f_q\}$, if $ \forall q \in Q$,

$$ \left | cost(X,q) - cost(C,q) \right | \le \epsilon\; cost(X,q). $$ 
\end{definition}
For coresets with a $\theta$ additive error, the above guarantee becomes the following-- for all $q$. 
$$ \left | cost(X,q) - cost(C,q) \right | \le \epsilon\; cost(X,q) + \theta. $$ 
Though stated in terms of decomposable loss functions, this definition also applies to measures which are not decomposable but are rather functions of entire dataset as a whole. We will also be dealing with weak coresets, i.e. our coresets will satisfy the above guarantees for a specific subset of the solution space, that contains the optimal. 

% Sensitivity of a point $x$ is defined as,
% $$\sigma(x) = \underset {q \in Q}{sup} \hspace{2pt} \frac{\mu_{X}(x) f_{q}(x)}{\sum_{{x}' \in X} \mu_{X}({x}') f_{q}({x}')}$$

In this text, we mainly deal with uniform sampling, and with two non-decomposable loss functions, the F1 score and the MCC score. For a particular classifier $q$, let $tp(q), tn(q), fp(q)$ and $fn(q)$ denote the true-positive, the true-negative, the false positive and false negative respectively.
We will sometimes drop the $(q)$ notation when it is clear from context.

The F1 score is defined as the following. For a classifier $q$, 
$$F_1(q) = 2 \cdot \frac{precision \cdot recall}{precision + recall } = \frac{tp}{tp + \frac{1}{2} (fp + fn)}$$

The MCC score is defined as the following (we leave off the "$(q)$" for lack of space). 
\begin{align*}
MCC(q) &= \frac{tp \cdot tn  - fp \cdot fn}{\sqrt{(tp + fp)(tp + fn)(tn + fp) (tn + fn)}}\\
 &= \frac{\frac{tp}{n} - T^{'}P^{'}}{\sqrt{T^{'} P^{'} (1- T^{'}) (1- P^{'})}}
\end{align*}
where $T'$ denotes the fraction of ground truth positives and $P' = (tp + fp)/n$ denotes the fraction of predicted positives. We will use $Y^+$ and $Y^-$ to denote the set of all positive and negatively labeled points. Sometimes we will also use the same notation to denote the total number of positive and negatively labeled points.  

% Non-decomposable loss functions such as F-measure, MCC(Matthew's correlation coefficient), ROC, AUC, precision@k are useful in sensitive domains such as medicine and biometrics, etc. The reason for that is compared to point loss functions such as cross-entropy or hinge-loss, these offer much more fine-grained control over prediction, but at the same time, it poses novel challenges in terms of algorithm design and analysis \cite{online_sgd_nondecomposable}.

% The classical way of defining sensitivity using importance sampling will not work here as we can not decompose these loss functions in terms of the sum of the loss of individual data points, and that makes this problem challenging to design a coreset of these loss functions.

% \textcolor{blue}{-- why direct optimize non decomposable measures ? this should be before classical way of defining sensitivity paragraph....
% }

\section{Related Work}

\cite{joachims2005support} formulated the non-decomposable measures as a structural prediction problem and gave a direct optimization algorithm using multivariate SVM formulation. The primary idea is to reduce the number of constraints in multivariate SVM by finding the most violated constraint in each iteration, preparing the candidate set, and applying multivariate SVM over that small set only. But in their algorithm for calculation of the \textit{argmax} for most violated constraints, the algorithm must go over all possible configurations of the contingency table, which amounts to $n^2$ different contingency tables.

\cite{online_sgd_nondecomposable} extend the existing online learning models for point-loss functions to non-decomposable loss functions. They also develop scalable SGD solvers for non-decomposable loss functions. Their work uses the previous result of representing the non-decomposable loss function as structural SVMs~\cite{joachims2005support} and optimizing them. Deep learning based methods for optimizing non-decomposable losses have also been explored in \cite{opt_nondecompp_deep_network}. We note that our work on coresets is orthogonal to the work on optimizing such losses, any such optimizer could be used in parallel with our coreset technique.

\cite{Google_scalable_nondecomposable} proposed an alternative formulation based on simple direct bounds on per-sample quantities indicating whether each sample is a true positive or a false positive. Using these bounds, they constructed global bounds on ranking-based measures such as precision at fixed recall, recall at fixed precision, and F1 score, among others. From these global bounds, they derived the surrogate objective function and the closed-form alternate optimization problem, which is then optimized.  However, they did not discuss empirical results for the F1 score and MCC. \cite{benedict2022sigmoidf1} try to maximize a surrogate loss in place of the F1 score. There has been some work in active sampling to estimate F1 score using optimal subsampling \cite{sawade2010active} or iterative importance sampling  \cite{poms2021low} ,however, both the motivation and guarantees are different from the coreset guarantees.
%A possible reason for this might be that their algorithm requires a constrained optimization problem to be solved that is challenging over deep networks.

Coresets have been studied since \cite{agarwal2005geometric}. One of the popular ways to create coreset is to construct a probability distribution, called {\em sensitivity},
over the set of input points~\cite{langberg2010universal, feldman2011unified}. Sampling points proportional to their sensitivities and appropriately reweighing them gives a coreset with high probability~\cite{langberg2010universal, feldman2011unified}. However, the major technical challenge in this method is to figure out computationally inexpensive upper bounds to the sensitivity scores for each cost function. More details about the construction of coresets using the sensitivity framework can be found in~\cite{ feldman2011unified,bachem2017practical,bachem2018scalable,braverman2016new,classification_coreset,feldman2020core} and references within. While coresets using this framework have theoretical guarantees , it has been shown that for some problems  you can not do much better than uniform sampling \cite{samadian2020unconditional}. Our paper also has similar results. It has been shown by \cite{braverman2022power} that using uniform sampling with some careful analysis can also give coreset guarantees for clustering problems. Additionally empirically uniform sampling has been shown to be comparable to other techniques for many real data sets \cite{lu2023coreset}.

% -- ~\cite{direct_loss_min}, \cite{Google_scalable_nondecomposable} and \cite{schafer2018dyad} are focused on ranking measures whereas our work addresses binary classification.

\section{Lower Bounds}
In this section we first present lower bounds for strong coresets for $F_1$ score and MCC. 

\subsection{Lower bound for Strong $F_1$ coresets}

\begin{theorem}
Let $F_1(q)$ be the $F_1$ score on full dataset $D$ with fixed query $q$ and $\tilde{F_1}(q)$ be the $F_1$ score on coreset $C$. Then, there does not exist a strong coreset $C$ with a size less than $n$ that satisfies relative error approximation i.e. $\tilde{F_1}(q) \in (1 \pm  \epsilon) F_1(q)$ for all queries $q \in Q$. 

\end{theorem}

\begin{proof}

Let $d$ be any non-negative even integer. Let ${\cal S}$ be the set of all sets of size $d/2$. We define the set $P$ of points as follows: In order to generate each point $p$, choose $B_p \in {\cal S}$, and $y_p \in \{\pm 1\}$. The vector $p$ is defined as following--- $p_i = -y_p$ for all $i \in B_p$, and $p_i = 0$ for all $i\in [d]\setminus B_p$, $p_{d+1} = y_p/2$. The label of the point $p$ is also $y_p$.

For each point $p$, we also define a corresponding classifier $w_p$ in the following way--- $w_p[i] = 0$ if $i\in B_p$, $w_p[i] = 1$ if $i \in [d] \setminus B_p$, and $w_p[d +1] = 1$. 

By the above construction, for every point 
\begin{align*}
p \cdot w_p = y_p/2. 
\end{align*}
Hence $sign(p \cdot w_p)   = y_p$, and hence $p$ is correctly classified. For any point $q \neq p$, $B_p \neq B_q$. Hence, 
\begin{align*}
    q\cdot w_p = -y_q |B_q \cap \bar{B_p}|  + y_q / 2 
\end{align*}
Since $|B_q \cap \bar{B_p}| \ge 1$, $sign(q \cdot w_p)=  - y_q$, i.e. $q$ is misclassified.  

So for a point $p$ with $y_p = +1$, for the query $w_p$ the $F_1$ score is $\frac{1}{1 + (n-2)/2}$. However, if the coreset does not sample the point $p$, the $F_1$ score for $w_p$ would be zero, which is not a relative error approximation. Since this is true for every point $p$, the coreset has to be of size $n$. 

\end{proof}

%%%Final modified proof for MCC lower bound

\subsection{Lower bound for strong $MCC$ coresets}

\begin{theorem}
Let $MCC(q)$ be the $MCC$ score on full dataset $D$ with fixed query $q$ and $\widetilde{MCC}(q)$ be the $MCC$ score on coreset $C$. Then, there does not exist a strong coreset $C$ with a size less than $n$ that satisfies relative error approximation i.e. $\widetilde{MCC}(q) \in (1 \pm  \epsilon) MCC(q)$ for all queries $q \in Q$. 

\end{theorem}

\begin{proof}

Let $d\ge4$ be any non-negative even integer. Let ${\cal S}$ be the set of all sets of size $d/2$. Let $n = {d\choose d/2}$.
Let $y_p \in \{\pm 1\}$ be such that exactly $n/2$ points have $y_p = +1$. We define the set $P$ of $n$ points as follows: in order to generate each point $p\in \mathbb{R}^{d+1}$, choose $B_p \in {\cal S}$. The vector $p$ is defined as follows: $p_i = X_p$ for all $i \in B_p$, and $p_i = 0$ for all $i\in [d]\setminus B_p$, $p_{d+1} = y_p/2$. The label of the point $p$ is also $y_p$. 

%%lets define X_p for all point p.

Now, lets define $X_p$ for points, $1 \le p \le n$ as follows. If the true label of point $p$ is $y_p \in Y^+$ then,
$X_p = y_p$ for $ 1 \le p \le \frac{Y^+}{2}$, and $-y_p$ else. Similarly, for $p \in Y^-$, $X_p = y_p$ for $1 \le p \le \frac{Y^-}{2}$ and $-y_p$ else. 

% $$
% X_p = \left\{\begin{matrix}
% y_p &; 1 \le p \le \frac{Y^+}{2}\\ 
% -y_p &; \frac{Y^+}{2} < p \le Y^+
% \end{matrix}\right. $$
% If 
% $$
% X_p = \left\{\begin{matrix}
% y_p &; 1 \le p \le \frac{Y^-}{2}\\ 
% -y_p &; \frac{Y^-}{2} < p \le Y^-
% \end{matrix}\right. $$

For each point $p$, we also define a corresponding classifier $w_p$ in the following way--- $w_p[i] = 0$ if $i\in B_p$, $w_p[i] = 1$ if $i \in [d] \setminus B_p$, and $w_p[d +1] = 1$. 

By the above construction, for every point 
\begin{align*}
p \cdot w_p = y_p/2. 
\end{align*}
Hence $sign(p \cdot w_p)   = y_p$, and hence $p$ is correctly classified. For any point $q \neq p$, $B_p \neq B_q$. Hence, 
\begin{align*}
    q\cdot w_p = X_q |B_q \cap \bar{B_p}|  + y_q / 2 
\end{align*}
Since $|B_q \cap \bar{B_p}| \ge 1$, $sign(q \cdot w_p)= X_q$, i.e. $q$ is classified as per the sign of $X_q$.

Therefore, for the classifier $w_p$, all of the points, except point $p$, classify according to $X_q$.

%%% if we pick point p then what is MCC and if we ignore that then what is mcc ... argument...

Notice that we have designed $X_p$ for all $p$ such that for $Y^+$ and $Y^-$, it will classify half points positive and half points as negative. 

Therefore for a point $p$ with $y_p = +1$ and classifier $w_p$,  we have $ tp = \frac{n-1}{4} + 1$, $tn = \frac{n-1}{4}$, $ fn = \frac{n-1}{4}$ and $ fp = \frac{n-1}{4}$ respectively for balanced setup i.e. $Y^+ = Y^- = \frac{n}{2}$.

For above setup $MCC$ score is,

\begin{align*}
   MCC & = \frac{tp \cdot tn  - fp \cdot fn}{\sqrt{(tp + fp) (tp + fn) (tn + fp) (tn + fn)}} \\
   &= \frac{tp \cdot tn  - fp \cdot fn}{\sqrt{(tp + fp) Y^+ Y^- (tn + fn)}} \\
   &= \frac{(\frac{n-1}{4} + 1) \cdot (\frac{n-1}{4})  - (\frac{n-1}{4}) \cdot (\frac{n-1}{4})}{\frac{n}{2}\sqrt{(\frac{n+1}{2})(\frac{n-1}{2})}} \\
   & = \theta \left(\frac{1}{n} \right )
\end{align*}

 However, if the coreset does not sample the point $p$, then we have  $ tp = \frac{n-1}{4} $, $tn = \frac{n-1}{4}$, $ fn = \frac{n-1}{4}$ and $ fp = \frac{n-1}{4}$ respectively, thus $MCC$ score for $w_p$ would be zero, which is not a relative error approximation. Since this is true for every point $p$, the strong coreset has to be of size $n$.
  
\end{proof}

%Stratified case for f1 score starts here... Final for all query one...

\section{Weak coreset for queries with high $F_1$ score}

We will be using the following result of \cite{li2001improved} to get our coreset guarantees for all queries in the query set of our interest.

\begin{theorem}\label{Thm : Li}\cite{li2001improved}

Let, $\alpha > 0$, $v>0$ and $\delta > 0$. Fix a countably infinite domain $X$ and let $q(\cdot)$ be any probability distribution over $X$. Let $F$ be a set of functions from $X$ to $\left[0, 1 \right]$ with $Pdim(F) = d^{'}$. Denote by $C$ a sample of $m$ points from $X$ sampled independently according to $q(\cdot)$. Then, for $m \in \Omega\left( \frac{1}{\alpha^2\cdot v} \left (d^{'}\log(1/v) + \log(1/\delta) \right) \right)$, with probability at least $1- \delta$ it holds that,

$$\forall f \in F; d_v \left(\sum_{x \in X} q(x) f(x), \frac{1}{|C|} \sum_{x \in C}f(x)\right ) \le \alpha$$

where, $d_v(a, b) = \frac{|a - b|}{a+b+v}$.

\end{theorem}

Now we state our theorem for bounding the coreset size and performance for $F_1$ score. 

\begin{theorem}
Let $\epsilon > 0$ and $c > 1$. Consider an instance where number of positive samples are $Y^+$ and number of negative samples are $Y^-$, and $n = Y^+ + Y^-$. We consider $Q_{\gamma}$ to be the set of queries such that $F_1(q) \ge \gamma$ and $tp\ge max\left( \frac{n(1 - c \cdot \epsilon)}{2c \cdot(1 - \epsilon)}, \frac{n(1 + c \cdot \epsilon)}{2c \cdot(1 + \epsilon)} \right)$ for $q\in Q_{\gamma}$. Let $d = $ vc-dimension$(Q_{\gamma})$.

Stratified uniform sampling with a total of $ \left( \frac{(2-\gamma)^2}{\gamma^2 \epsilon^2} + \frac{1}{\epsilon^2} + \left (\frac{Y^-}{Y^+} \right)^{2} \frac{1}{\epsilon^2} \right ) \cdot \left(d + \log\frac{1}{\delta} \right) $ samples would be able to give a coreset for $Q_{\gamma}$ that satisfies $\tilde{F_1}(q) \in (1 \pm  c \cdot\epsilon) F_1(q)$ for all queries $q \in Q_{\gamma}$ with probability at least $1-3\delta$, for a suitable $c > 1$.

\end{theorem}

\begin{proof}
    
We have, $\frac{tp}{tp + \frac{1}{2}(fp+fn)} = F_1 \ge \gamma$, and hence 
$\frac{tp}{\frac{1}{2}(fp+fn)} \ge \frac{\gamma}{1 - \gamma}$.  Also, $tp + fp+fn \ge Y^+$. We now find the minimum value of $tp$ satisfying these two inequalities. 

 From above, $tp = \left(\frac{\gamma}{1-\gamma} \right) \cdot \frac{1}{2}(fp+fn) + k$, where $k \ge 0$. Then, $\left(\frac{\gamma}{1-\gamma} \right) \frac{1}{2}(fp+fn) + k + (fp+fn) \ge Y^+$ and hence $\frac{1}{2}(fp+fn) \ge (Y^+ - k)/\left(\frac{\gamma}{1-\gamma} + 2\right)$. Thus $tp = \frac{\left(\frac{\gamma}{1-\gamma} \right ) Y^+ }{\left(\frac{\gamma}{1-\gamma} + 2\right)} + \frac{2k}{\frac{\gamma}{1-\gamma} + 2} \ge \frac{\gamma \cdot Y^+}{(2 - \gamma)}$. \\

%%%  tp approximation and sample size ...

Let $\widetilde{tp}$ and $\widetilde{fn}$ be the true positive and false negative counts obtained from the sampled set. Let all the negative samples be given a weight of $\frac{Y^-}{Y^+}$. Hence, let $\widetilde{\widetilde{fp}}$ defined as following---
$\widetilde{\widetilde{fp}} = \left (\frac{Y^-}{Y^+} \right) \widetilde{fp}.$
Similarly, let $\widetilde{\widetilde{tn}}$ defined as following--- $\widetilde{\widetilde{tn}} = \left (\frac{Y^-}{Y^+} \right) \widetilde{tn}. $
Our algorithm uses the four estimates 
$\widetilde{tp}$, $\widetilde{\widetilde{fp}}$, $\widetilde{\widetilde{tn}}$ and $\widetilde{fn}$ to estimate the F1 score. We first show the quality of each of these estimates by applying Theorem \ref{Thm : Li} individually. 

For $tp$ approximation,
$$\forall q \in Q; d_v \left(\sum_{x \in Y^+}\frac{1}{Y^+} \delta_x, \frac{1}{|S_1|} \sum_{x \in S_1}\delta_x\right ) \le \alpha_1$$

where, $\delta_x = {1}(x \in tp)$ is indicator variable and $d_v(a, b) = \frac{|a - b|}{a+b+v}$

\begin{align*}
     &= \frac{\left | \sum_{x \in Y^+}\frac{1}{Y^+} \delta_x - \frac{1}{|S_1|} \sum_{x \in S_1}\delta_x \right |}{ \sum_{x \in Y^+}\frac{1}{Y^+} \delta_x + \frac{1}{|S_1|} \sum_{x \in S_1}\delta_x + v} \le \alpha_1\\
     & = \frac{\left | \frac{tp}{Y^+} - \frac{\widetilde{tp}}{|S_1|} \right |}{ \frac{tp}{Y^+} + \frac{\widetilde{tp}}{|S_1|} + v} \le \alpha_1 \\
     & = \left | \frac{tp}{Y^+} - \frac{\widetilde{tp}}{|S_1|} \right | \le 3 \alpha_1
\end{align*}

Since, $\frac{tp}{Y^+}$ and $\frac{\widetilde{tp}}{|S_1|}$ is less than one, lets take $v = \frac{1}{2}$.

$$\widetilde{tp} \in tp \left (\frac{S_1}{Y^+}\right) \pm S_1 \cdot (3\alpha_1)$$

$$ \frac{Y^+}{S_1}\widetilde{tp} \in tp \pm Y^+ \cdot (3\alpha_1)$$

Since, $tp \ge \frac{\gamma \cdot Y^+}{(2 - \gamma)}$, we have $Y^+ \le \frac{ tp \cdot(2 - \gamma)}{\gamma}$.

$$\frac{Y^+}{S_1}\widetilde{tp} \in tp \pm  (3\alpha_1) \cdot \frac{ tp \cdot(2 - \gamma)}{\gamma}$$

$$\frac{Y^+}{S_1}\widetilde{tp} \in \left ( 1 \pm  (3\alpha_1) \cdot \frac{(2 - \gamma)}{\gamma} \right )tp $$

$$\widetilde{tp} \in \left ( 1 \pm  \epsilon \right )tp \frac{S_1}{Y^+} $$

where, $\epsilon = (3\alpha_1) \cdot \frac{(2 - \gamma)}{\gamma}$ \\

Therefore, $\alpha_1 = \frac{\gamma \cdot \epsilon}{3(2-\gamma)}$.

Now, size of samples required to satisfy the above approximation using Theorem 5.1 is,
$S_1 = \Omega\left( \frac{1}{\alpha_1^2 \cdot v} \left ( d \log \frac{1}{v} + \log \frac{1}{\delta} \right)\right) = \Omega\left( \frac{(2 - \gamma)^2}{\gamma^2 \cdot \epsilon^2} \left ( d  + \log \frac{1}{\delta} \right)\right) $, with probability $1 - \delta$ for all $q \in  Q_\gamma \subset Q$ . \\

%%% end of tp calculation...

%%% fn approximation and sample size ...

For $fn$ approximation, we can similarly show that 
$$\widetilde{fn} \in fn \left (\frac{S_2}{Y^+}\right) \pm S_2 \cdot (3\alpha_2).$$
Using $\epsilon = (3\alpha_2)$, we have the additive error to be $\epsilon S_2$. \\

% $$\forall q \in Q; d_v \left(\sum_{y \in Y^+}\frac{1}{Y^+} \delta_y, \frac{1}{|S_2|} \sum_{y \in S_2}\delta_y\right ) \le \alpha_2$$

% where, $\delta_y = 1(y \in fn)$ and $d_v(a, b) = \frac{|a - b|}{a+b+v}$

% \begin{align*}
%      &= \frac{\left | \sum_{y \in Y^+}\frac{1}{Y^+} \delta_y - \frac{1}{|S_2|} \sum_{y \in S_2}\delta_y \right |}{ \sum_{y \in Y^+}\frac{1}{Y^+} \delta_y + \frac{1}{|S_2|} \sum_{y \in S_2}\delta_y + v} \le \alpha_2\\
%      & = \frac{\left | \frac{fn}{Y^+} - \frac{\widetilde{fn}}{|S_2|} \right |}{ \frac{fn}{Y^+} + \frac{\widetilde{fn}}{|S_2|} + v} \le \alpha_2 \\
%      & = \left | \frac{fn}{Y^+} - \frac{\widetilde{fn}}{|S_2|} \right | \le 3 \alpha_2
% \end{align*}

% Since, $\frac{fn}{Y^+}$ and $\frac{\widetilde{fn}}{|S_2|}$ is less than one, lets take $v = \frac{1}{2}$.

% $$\widetilde{fn} \in fn \left (\frac{S_2}{Y^+}\right) \pm S_2 \cdot (3\alpha_2)$$ \\

Now, size of samples required to satisfy the above approximation using Theorem 5.1 is,
$S_2 = \Omega\left( \frac{1}{\alpha_2^2 \cdot v} \left ( d \log \frac{1}{v} + \log \frac{1}{\delta} \right)\right) = \Omega\left( \frac{1}{\epsilon^2} \left ( d  + \log \frac{1}{\delta} \right)\right) $, with probability $1 - \delta$ for all $q \in Q_\gamma \subset Q$ . \\

%%% end of fn calculation...

%%% fp approximation and sample size ...

For $fp$ approximation, 

$$\forall q \in Q; d_v \left(\sum_{z \in Y^-}\frac{1}{Y^-} \delta_z , \frac{1}{|S_3|} \sum_{z \in S_3}\delta_z \right ) \le \alpha_3$$

where, $\delta_z = 1(z \in fp)$ and $d_v(a, b) = \frac{|a - b|}{a+b+v}$. Again, using similar arguments, we can show that 
$$\widetilde{fp} \in fp \left (\frac{S_3}{Y^-}\right) \pm S_3 \cdot \epsilon,$$
where again $\epsilon \ge 3 \alpha_3$.
% \begin{align*}
%      &= \frac{\left | \sum_{z \in Y^-}\frac{1}{Y^-} \delta_z - \frac{1}{|S_3|} \sum_{z \in S_3}\delta_z \right |}{ \sum_{z \in Y^-}\frac{1}{Y^-} \delta_z + \frac{1}{|S_3|} \sum_{z \in S_3}\delta_z + v} \le \alpha_3\\
%      & = \frac{\left | \frac{fp}{Y^-} - \frac{\widetilde{fp}}{|S_3|} \right |}{ \frac{fp}{Y^-} + \frac{\widetilde{fp}}{|S_3|} + v} \le \alpha_3 \\
%      & = \left | \frac{fp}{Y^-} - \frac{\widetilde{fp}}{|S_3|} \right | \le 3 \alpha_3
% \end{align*}

% Since, $\frac{fp}{Y^-}$ and $\frac{\widetilde{fp}}{|S_3|}$ is less than one, lets take $v = \frac{1}{2}$.
% $$\widetilde{fp} \in fp \left (\frac{S_3}{Y^-}\right) \pm S_3 \cdot (3\alpha_3)$$

For $\widetilde{\widetilde{fp}}$, we get the following bound.

$$\widetilde{\widetilde{fp}} = \left (\frac{Y^-}{Y^+} \right) \widetilde{fp} \in fp \left (\frac{S_3}{Y^+}\right) \pm S_3 \cdot (3\alpha_3) \left (\frac{Y^-}{Y^+} \right) $$ \\

% Notice that this can be achieved by  reweighing all negative class data points by $ \left (\frac{Y^+}{Y^-} \right)$ and thus this term will also comes in the calculation of the size to satisfy $\widetilde{\widetilde{fp}}$. 

After applying reweighing our modified $\alpha_3$ is defined as, $\alpha_{3}^{'} = \frac{\epsilon}{3} \left (\frac{Y^+}{Y^-} \right)$. \\

Now, size of samples require to satisfy the above approximation using Theorem 5.1 is,
$S_3 = \Omega\left( \frac{1}{(\alpha_{3}^{'})^2 \cdot v} \left ( d \log \frac{1}{v} + \log \frac{1}{\delta} \right)\right) = \Omega\left( \left (\frac{Y^-}{Y^+} \right)^{2} \frac{1}{\epsilon^2} \left ( d  + \log \frac{1}{\delta} \right)\right) $, with probability $1 - \delta$ for all $q \in Q$ where $Q_\gamma \subset Q$ . \\

%%% end of fp calculation...

Thus, in-order to satisfy all three approximations viz. $tp$, $fn$ and $fp$ we would require total sample size, 

\begin{align*}
S &= S_1 + S_2 + S_3\\
%&=\frac{(2 - \gamma)^2}{\gamma^2 \cdot \epsilon^2} \left ( d  + \log \frac{1}{\delta} \right) +  \frac{1}{\epsilon^2} \left ( d  + \log \frac{1}{\delta} \right) +  \left (\frac{Y^-}{Y^+} \right)^{2} \frac{1}{\epsilon^2} \left ( d  + \log \frac{1}{\delta} \right)\\
& = \left( \frac{(2-\gamma)^2}{\gamma^2 \epsilon^2} + \frac{1}{\epsilon^2} + \left (\frac{Y^-}{Y^+} \right)^{2} \frac{1}{\epsilon^2} \right ) \cdot \left(d + \log\frac{1}{\delta} \right) 
\end{align*}

Thus, with probability $ 1- 3\delta$, above approximation for $tp$, $fn$ and $fp$ holds. 

Now, for the coreset left hand side guarantee,

\begin{align*}
    \tilde{F_1} & = \frac{\tilde{tp}}{\tilde {tp}+ \frac{1}{2} (\tilde{fn} + \widetilde{\widetilde{fp}})} \\
    &\ge \frac{(1 - \epsilon)  \frac{S \cdot tp}{Y^+}}{(1 - \epsilon) \frac{S \cdot tp}{Y^+} + \frac{1}{2} \left [\frac{S \cdot fn}{Y^+} + \frac{S \cdot fp}{Y^+} + (\epsilon \cdot S) \left( 1 + \frac{Y^-}{Y^+} \right) \right]} \\
    &= \frac{(1 - \epsilon) tp}{(1 - \epsilon) tp + \frac{1}{2} \left [fn + fp + (\epsilon) \left( Y^+ + Y^- \right) \right]}\\
    &= \frac{(1 - \epsilon) tp}{(1 - \epsilon) tp + \frac{1}{2} \left [fn + fp + (\epsilon \cdot n) \right]} \\
    &= \frac{(1 - \epsilon) tp}{ tp + \frac{1}{2} (fn + fp) - \left [\epsilon \cdot tp - \frac{\epsilon \cdot n}{2}\right]} 
\end{align*}

If the number of samples belongs to $tp$ is chosen to satisfy $tp \ge \frac{n(1 - c \cdot \epsilon)}{2c \cdot(1 - \epsilon)}$, then we have that 

\begin{align*}
\tilde{F_1} &\ge \frac{(1 - \epsilon) tp}{ tp + \frac{1}{2} (fn + fp) -\left [\epsilon \cdot tp - \frac{\epsilon \cdot n}{2}\right]} \\
&\ge \frac{(1 - c \cdot \epsilon) tp}{ tp + \frac{1}{2} (fn + fp)} = (1 - c \cdot \epsilon) F_1
\end{align*}

Similarly, for the coreset right hand side guarantee, 

\begin{align*}
     \tilde{F_1} & = \frac{\tilde{tp}}{\tilde {tp}+ \frac{1}{2} (\tilde{fn} + \widetilde{\widetilde{fp}})} \\
    &\le \frac{(1 + \epsilon)  \frac{S \cdot tp}{Y^+}}{(1 + \epsilon) \frac{S \cdot tp}{Y^+} + \frac{1}{2} \left [\frac{S \cdot fn}{Y^+} + \frac{S \cdot fp}{Y^+} - (\epsilon \cdot S) \left( 1 + \frac{Y^-}{Y^+} \right) \right]} \\
    &= \frac{(1 + \epsilon) tp}{(1 + \epsilon) tp + \frac{1}{2} \left [fn + fp - (\epsilon) \left( Y^+ + Y^- \right) \right]}\\
    &= \frac{(1 + \epsilon) tp}{(1 + \epsilon) tp + \frac{1}{2} \left [fn + fp - (\epsilon \cdot n) \right]} \\
    &= \frac{(1 + \epsilon) tp}{ tp + \frac{1}{2} (fn + fp) + \left [\epsilon \cdot tp - \frac{\epsilon \cdot n}{2}\right]}
\end{align*}

Again, using the similar assumptions on $tp$, we have that for $tp \ge \frac{n(1 + c \cdot \epsilon)}{2c \cdot(1 + \epsilon)}$, we get

\begin{align*}
\tilde{F_1} &\le \frac{(1 + \epsilon) tp}{ tp + \frac{1}{2} (fn + fp) + \left [\epsilon \cdot tp - \frac{\epsilon \cdot n}{2}\right]}\\
 &\le \frac{(1 + c \cdot \epsilon) tp}{ tp + \frac{1}{2} (fn + fp)} = (1 + c \cdot \epsilon)F_1
\end{align*}

\end{proof}

\section{Weak coreset for $MCC$}

\begin{theorem}
Let $\epsilon > 0$. Consider an instance where number of positive samples are $Y^+$ and number of negative samples are $Y^-$, and $n = Y^+ + Y^-$. Let $T$ to be the ground truth positive $0/1$ labels and $P$ to be the predicted positive $0/1$ labels. Let $tp, fp, fn$ be the true positive, false positive and false negative on the full data, $T^{'} =  \frac{\sum_{i} T_i}{n} = \frac{tp + fn}{n} = \frac{\left | Y^+  \right |}{n}$ and $P^{'} =  \frac{\sum_{i} P_i}{n} = \frac{tp + fp}{n}$. We consider $Q_{\gamma}$ to be the set of queries such that $tp \ge \gamma \cdot n$ and $tn \ge \gamma \cdot n$  for $q\in Q_{\gamma}$. Let $d=vc-dimension\left( Q_\gamma \right)$. We claim that uniform sampling with $\left(\frac{1}{\epsilon^2} \left ( d  + \log \frac{1}{\delta} \right) \left( 2 + 2 \cdot \left (\frac{Y^-}{Y^+} \right)^{2}\right) \right)$ samples would be able to give a coreset for $Q_{\gamma}$ that satisfies $\frac{MCC(q)}{\left  ( 1 + \frac{\epsilon}{\gamma}\right)} - 2 \cdot \epsilon \cdot C \le \widetilde{MCC}(q) \le \frac{MCC(q)}{\left  ( 1 - \frac{\epsilon}{\gamma}\right)} + 2 \cdot \epsilon \cdot C^{'}$ for all queries $q \in Q_{\gamma}$ with probability at least $1-4\delta$, where $C \le \frac{ \left( \frac{1}{\gamma} \right)}{ \left (1 + \frac{\epsilon}{\gamma} \right) \sqrt{T^{'} (1- T^{'}) \cdot \gamma}}$ and $C^{'} \le \frac{ \left( \frac{1}{\gamma} \right)}{ \left (1 - \frac{\epsilon}{\gamma} \right) \sqrt{T^{'} (1- T^{'}) \cdot \gamma}}$.

\end{theorem}

The proof can be found in the appendix.

%%%%%% Final combine table for all three dataset timing..

\begin{table*}[h]
  \centering
  \begin{tabular}{|c||c||c||c|}
    \hline
    \textbf{Coreset Algorithm} & \textbf{CoverType Time (in sec)} & \textbf{Adult Time (in sec)} & \textbf{KDD Cup '99 Time (in sec)} \\
    \hline
        uniform & 0.01418  & 0.003462 & 0.01028 \\ \hline
        leverage score &  0.1331 & 0.01132 & 0.0909\\ \hline
        lewis score & 2.4918 & 0.19029 & 1.9072 \\ \hline
        k-means& 0.05942 & 0.007611 & 0.0400 \\
    \hline
  \end{tabular}
  \caption{Time taken to prepare coreset of 10\% of full dataset for different datasets.}
  \label{table 1}
\end{table*}

%%%%% Final combine table ends here....

\begin{comment}
    
%%%% Table for covertype dataset timing.

\begin{table}
  \centering
  \caption{Time taken to prepare coreset of 10\% of CoverType dataset}
  \begin{tabular}{|c||c|}
    \hline
    \textbf{Coreset Algorithm} & \textbf{Time (in sec)} \\
    \hline
        uniform & 0.01418\\ \hline
        leverage score &  0.1331\\ \hline
        lewis score & 2.4918\\ \hline
        k-means& 0.05942\\
    \hline
  \end{tabular}
\end{table}

%%%% Table for Adult dataset timing.

\begin{table}
  \centering
  \caption{Time taken to prepare coreset of 10\% of Adult dataset}
  \begin{tabular}{|c||c|}
    \hline
    \textbf{Coreset Algorithm} & \textbf{Time (in sec)} \\
    \hline
    
        uniform & 0.003462  \\ \hline
        leverage score & 0.01132  \\ \hline
        lewis score& 0.19029 \\  \hline
        k-means& 0.007611 \\
        
    \hline
  \end{tabular}
\end{table}

%%%% Table for KDD Cup 99 dataset timing.

\begin{table}
  \centering
  \caption{Time taken to prepare coreset of 10\% of KDD Cup '99 dataset}
  \begin{tabular}{|c||c|}
    \hline
    \textbf{Coreset Algorithm} & \textbf{Time (in sec)} \\
    \hline
    
        uniform & 0.01028  \\  \hline
        leverage score & 0.0909  \\  \hline
        lewis score& 1.9072  \\  \hline
        k-means& 0.0400 \\
        
    \hline
  \end{tabular}
\end{table}
    
\end{comment}

%%%%% figure plot start here....

\begin{figure*}[h]
     \centering
     \begin{subfigure}[b]{0.33\textwidth}
         \centering
         \includegraphics[width= \textwidth, height=0.21\textheight]{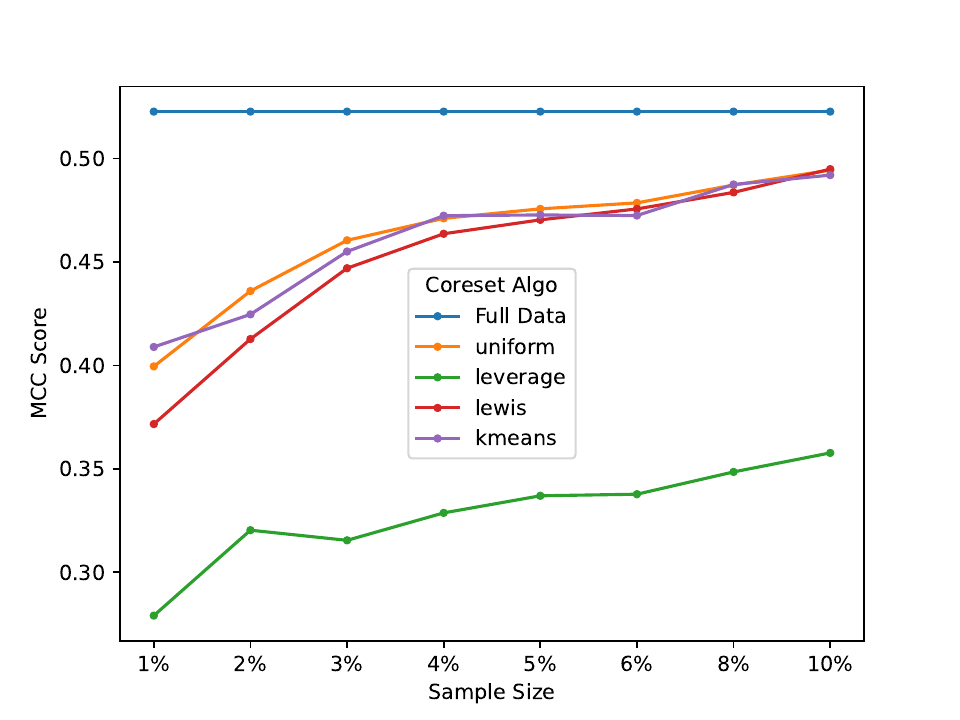}
         \caption{CovType Logistic}
         \label{fig:CovType Logistic MCC}
     \end{subfigure}
     \hfill
     \begin{subfigure}[b]{0.33\textwidth}
         \centering
         \includegraphics[width= \textwidth, height=0.21\textheight]{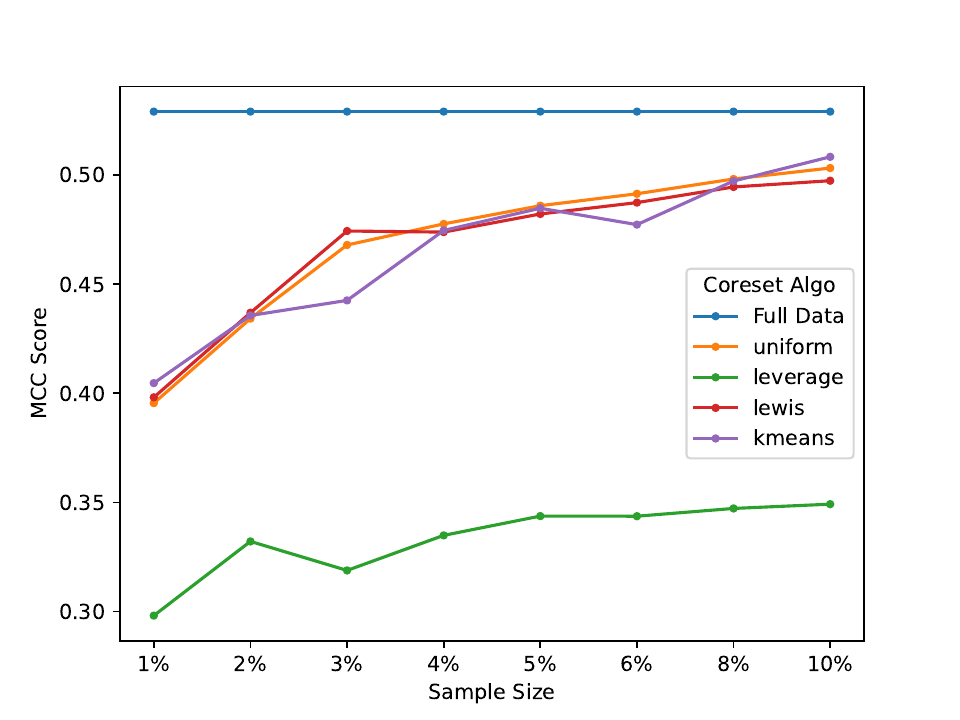}
         \caption{CovType SVM }
         \label{fig:CovType SVM MCC}
     \end{subfigure}
     \hfill
     \begin{subfigure}[b]{0.33\textwidth}
         \centering
         \includegraphics[width=\textwidth, height=0.21\textheight]{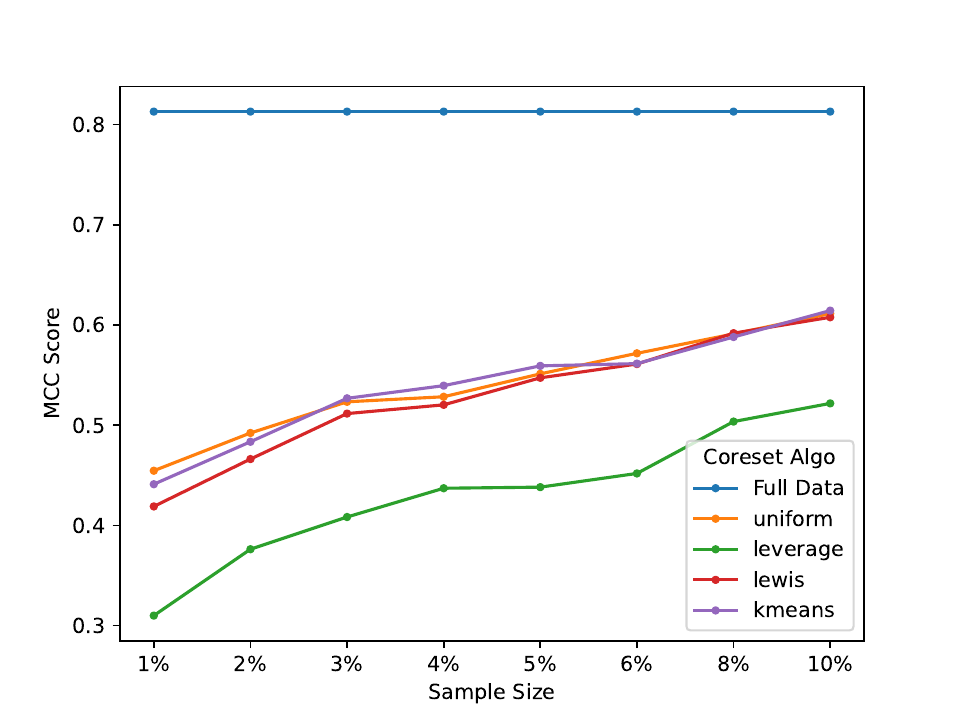}
         \caption{CovType MLP}
         \label{fig:CovType MLP MCC}
     \end{subfigure}
        \caption{$MCC$ Score on CovType}
        \label{fig:MCC Score on CovType}
\end{figure*}

\begin{figure*}[h]
     \centering
     \begin{subfigure}[b]{0.33\textwidth}
         \centering
         \includegraphics[width= \textwidth, height=0.25\textheight]{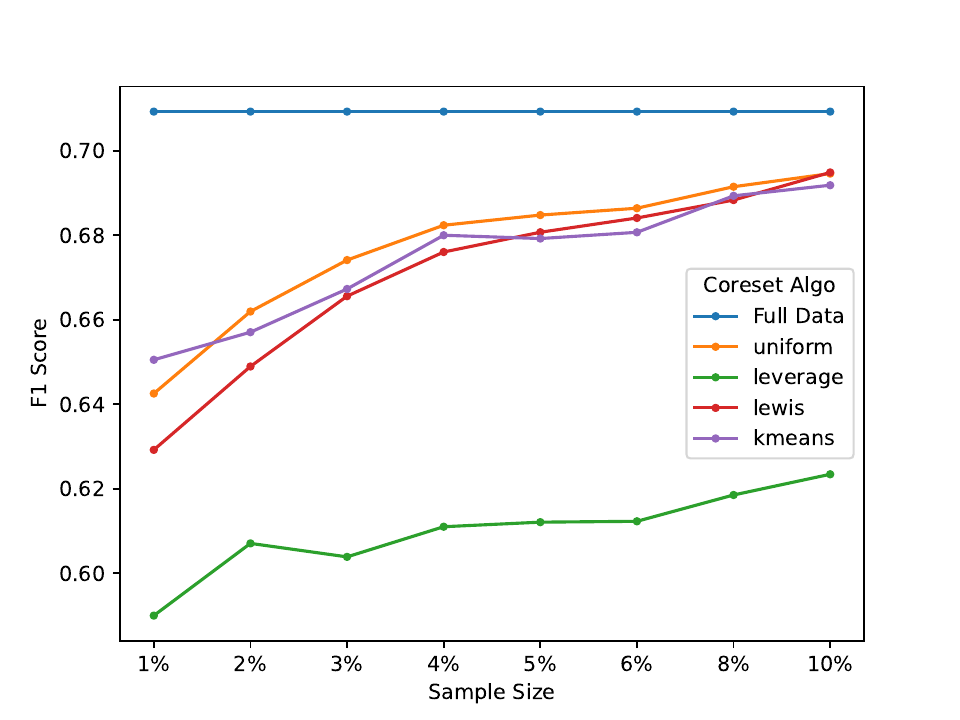}
         \caption{CovType Logistic}
         \label{fig:CovType Logistic F1}
     \end{subfigure}
     \hfill
     \begin{subfigure}[b]{0.33\textwidth}
         \centering
         \includegraphics[width=\textwidth, height=0.25\textheight]{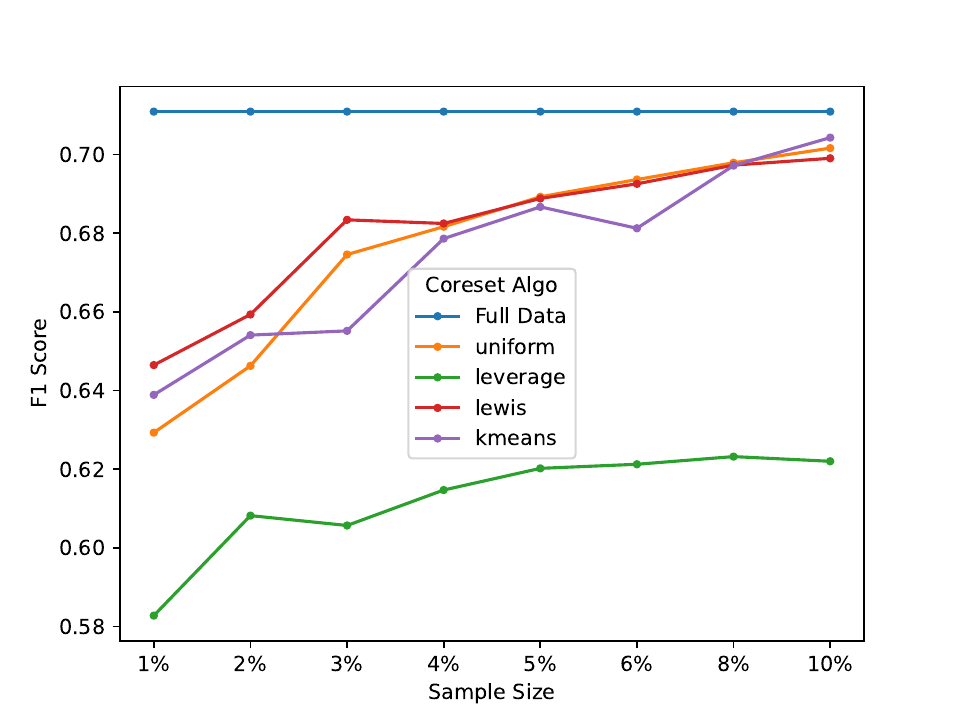}
         \caption{CovType SVM}
         \label{fig:CovType SVM F1}
     \end{subfigure}
     \hfill
     \begin{subfigure}[b]{0.33\textwidth}
         \centering
         \includegraphics[width=\textwidth, height=0.25\textheight]{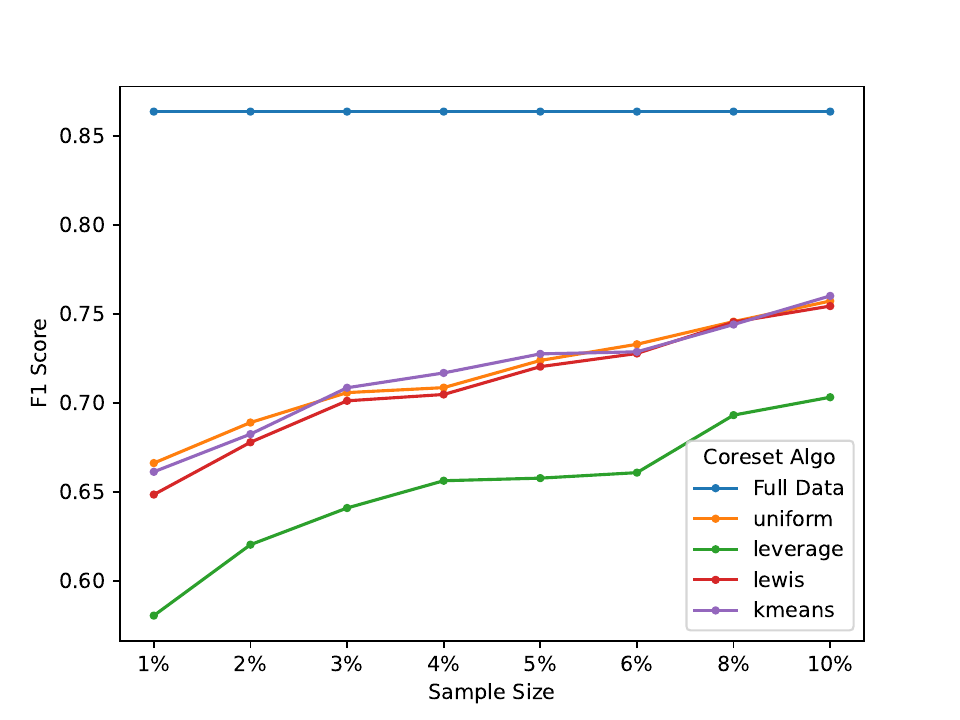}
         \caption{CovType MLP}
         \label{fig:CovType MLP F1}
     \end{subfigure}
        \caption{$F_1$ Score on CovType}
        \label{fig:F1 Score on CovType}
\end{figure*}

\begin{figure*}[h]
     \centering
     \begin{subfigure}[b]{0.33\textwidth}
         \centering
         \includegraphics[width=\textwidth, height=0.25\textheight]{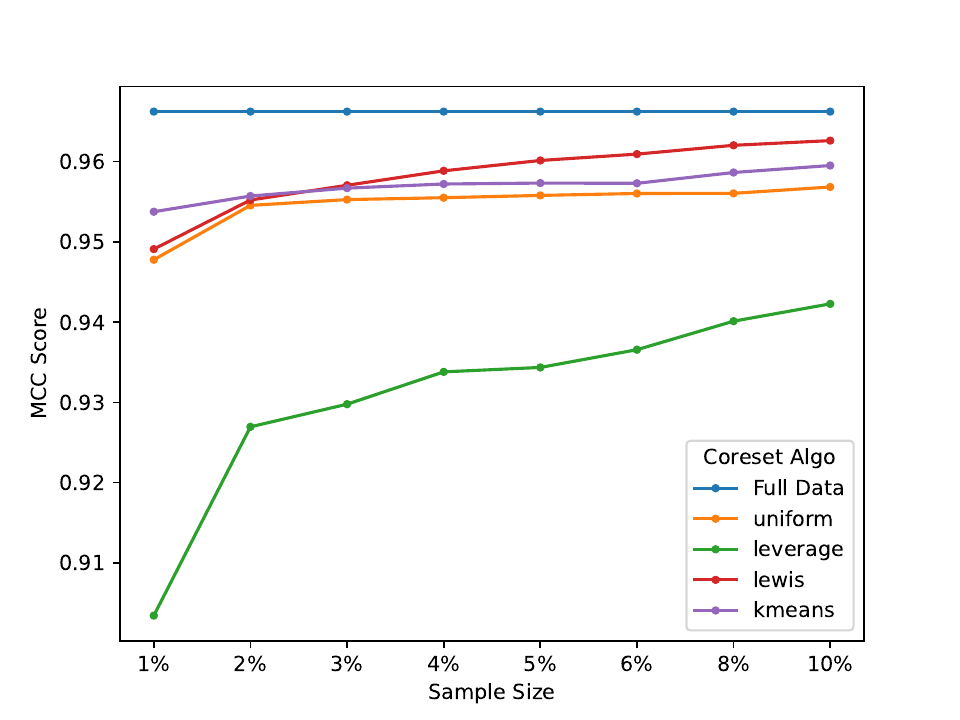}
         \caption{KDDCup Logistic}
         \label{fig:KDDCup Logistic MCC}
     \end{subfigure}
     \hfill
     \begin{subfigure}[b]{0.33\textwidth}
         \centering
         \includegraphics[width=\textwidth, height=0.25\textheight]{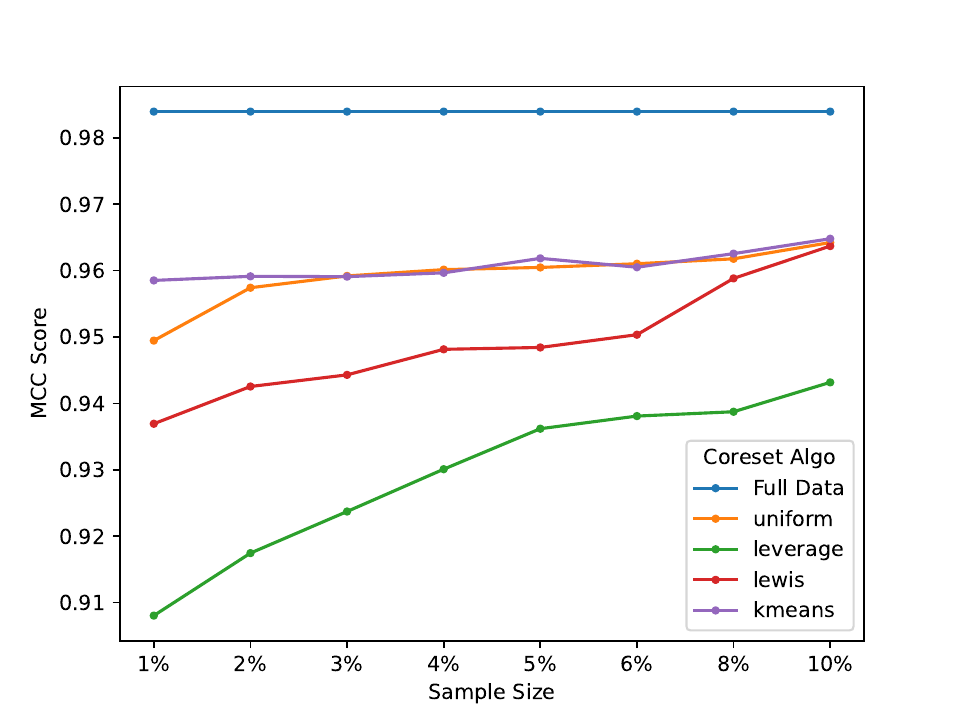}
         \caption{KDDCup SVM MCC}
         \label{fig:KDDCup SVM}
     \end{subfigure}
     \hfill
     \begin{subfigure}[b]{0.33\textwidth}
         \centering
         \includegraphics[width=\textwidth, height=0.25\textheight]{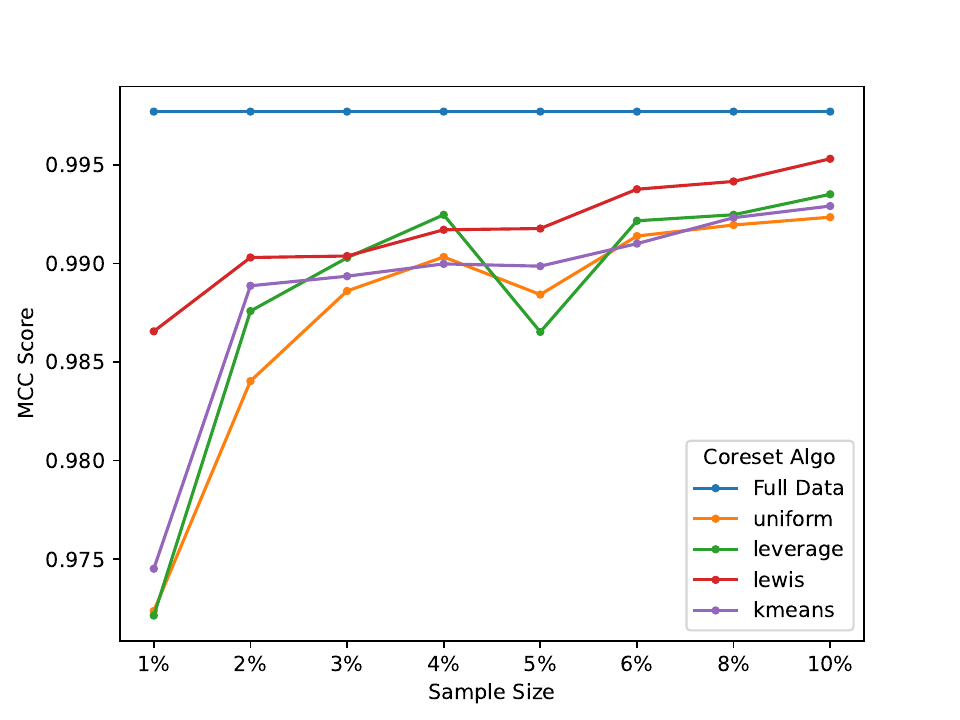}
         \caption{KDDCup MLP MCC}
         \label{fig:KDDCup MLP}
     \end{subfigure}
        \caption{$MCC$ Score on KDDCup}
        \label{fig:MCC Score on KDDCup}
\end{figure*}

\begin{figure*}[!h]
     \centering
     \begin{subfigure}[b]{0.33\textwidth}
         \centering
         \includegraphics[width=\textwidth, height=0.25\textheight]{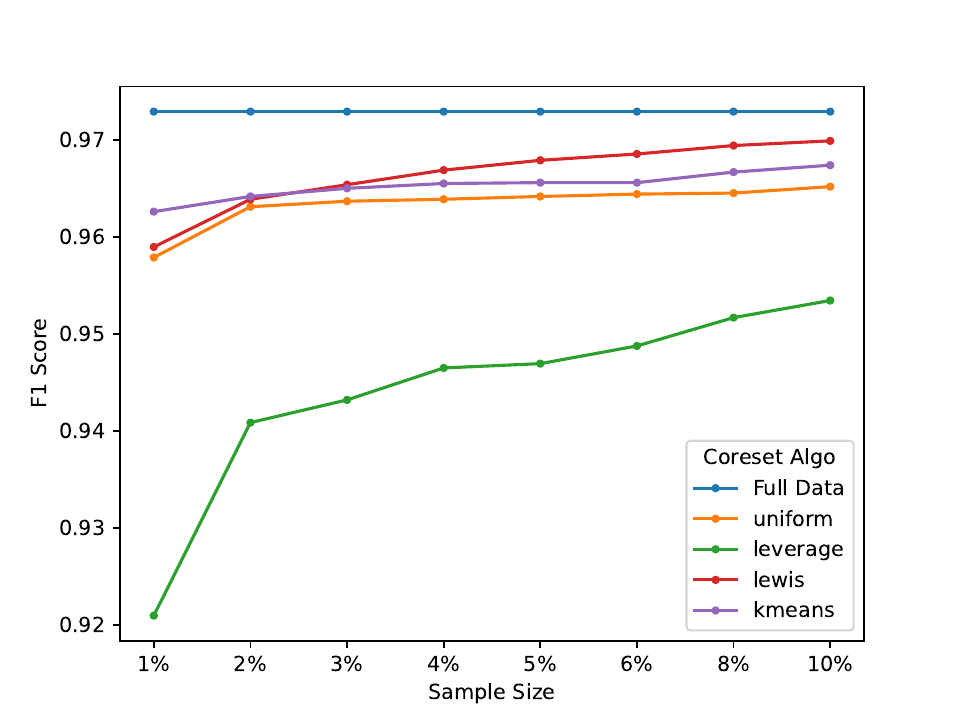}
         \caption{KDDCup Logistic}
         \label{fig:KDDCup Logistic F1}
     \end{subfigure}
     \hfill
     \begin{subfigure}[b]{0.33\textwidth}
         \centering
         \includegraphics[width=\textwidth, height=0.25\textheight]{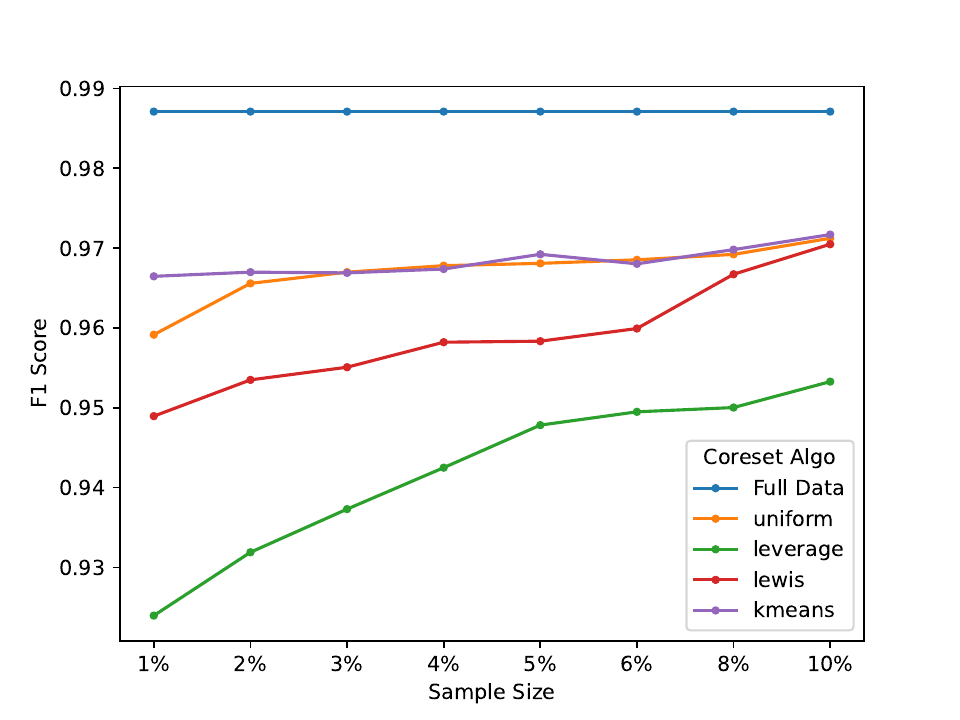}
         \caption{KDDCup SVM}
         \label{fig:KDDCup SVM F1}
     \end{subfigure}
     \hfill
     \begin{subfigure}[b]{0.33\textwidth}
         \centering
         \includegraphics[width=\textwidth, height=0.25\textheight]{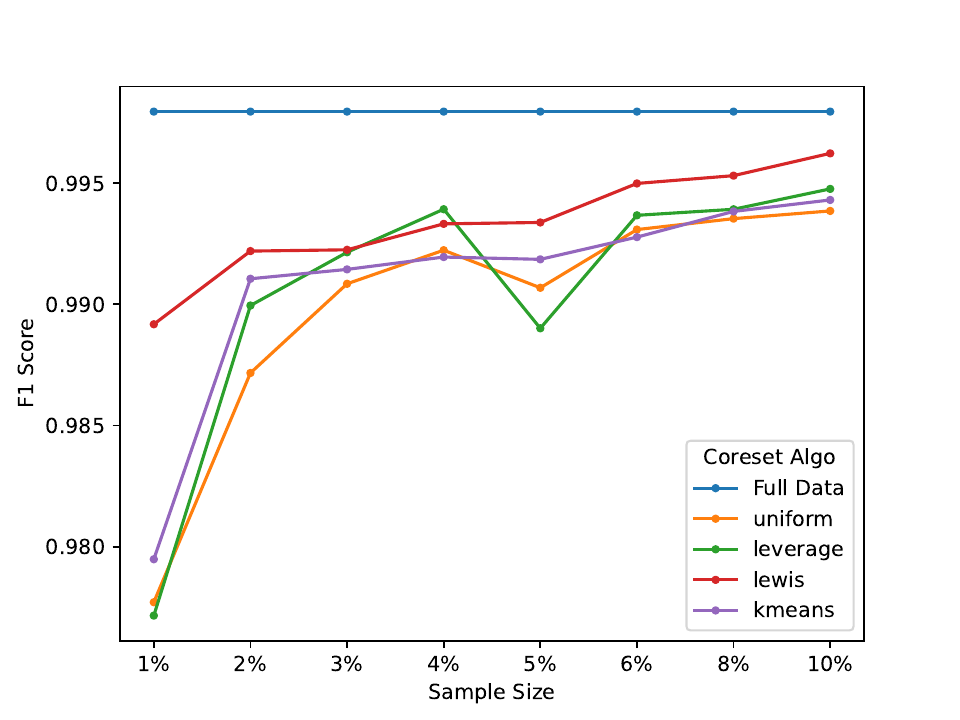}
         \caption{KDDCup MLP}
         \label{fig:KDDCup MLP F1}
     \end{subfigure}
        \caption{$F_1$ Score on KDDCup}
        \label{fig:F1 Score on KDDCup}
\end{figure*}

%%% figure plot ends here.....

\section{Experiments}

 All experiments were run on a computer with Nvidia Tesla V100 GPU with 32 GB memory and 28 CPUs. We used Python and its frameworks to implement our experiments. 

\paragraph{Data Sets:} The COVERTYPE \cite{misc_covertype_31} data consists of 581, 012 cartographic observations of different forests with 54 features. The task is to predict the type of trees at each location (49\% positive). We selected in a stratified way 50,000 samples from the real data and used them as our training data. The KDDCUP ’99 \cite{misc_kdd_cup_1999_data_130} data comprises of 494,021 network connections with 41 features, and the task is to detect network intrusions (20\% positive). In our experiments, we selected in a stratified way 50,000 samples from the real data and used them as training data. The Adult \cite{ misc_adult_2} dataset is a widely-used dataset containing information about individuals from the 1994 U.S. Census Bureau database. It consists of approximately 32,000 instances with 14 attributes, including age, education, and occupation. The dataset aims to predict whether an individual's income exceeds \$50,000 per year.

Our experiment created binary classification datasets by converting a multiclass dataset into a binary by flipping labels. Our final binary dataset class ratios for the CoverType dataset are [31770, 18230], for the Adult dataset [24294, 7706], and for the KDDCup dataset, it is [40154, 9846].  Here in $[a,b]$, $a$ is the number of points of positive class, and $b$ represents a number of points with negative class.

\paragraph{Experimental Assessment :} To verify our theoretical claims, we tested a uniform sampling coreset with some of the sophisticated coresets like leverage score \cite{drineas2006sampling}, $l_1$-lewis score \cite{cohen2015lp}, and k-means coreset by \cite{bachem2018scalable}.

%% model brief description...

As our classifier models, we used SVM classifier with linear kernel; vanilla logistic regression model from the sklearn library is used with default hyperparameters and a multilayer perceptron(MLP). For MLP experiments, we considered a simple MLP classifier with two hidden layers of size 100 each and the final output layer of size two, as we are dealing with binary classification. The optimizer used for the MLP is Adam, and the activation function used is ReLU.

%%% model description ends here....

We first prepared different coresets from the full datasets. Then we train our classifier models: logistic regression, SVM, and Feed Forward Neural Network on our coreset as well as on the full dataset.

For testing the performance of our coresets, we took models trained using the coresets and tested them on the entire training dataset. We report the evaluation measures (F1 and MCC) obtained and compare it with the original value obtained on full data(using model trained on full data). For all our experiments, plotted values are means taken over five independent repetitions of each experiment.

Figures \ref{fig:MCC Score on CovType} through \ref{fig:F1 Score on KDDCup} clearly show that uniform sampling gives superior or comparable performance to other sophisticated methods for  both F1 score and MCC. Also it can be seen that with increasing coreset size, the performance of model trained on the coreset also improves as expected. We also measure the time required to prepare coreset using these techniques on the different datasets. The times are reported in Table \ref{table 1}. It is clear that uniform sampling is many times faster than the other methods. Hence at much lower computation times we get better or comparable performances to other coreset construction strategies. Some additional experiments can be found in the appendix. 

\section{Conclusion}
We initiated the study of coresets for non-decomposable classification measures, specifically for the F1 score and MCC. We showed lower bounds for strong coresets and construction of weak coresets using stratified uniform sampling.  It would be interesting to see whether coresets with better additive guarantees and lesser assumptions on the query vector can be developed. Similarly, algorithm-specific subset selection strategies could be explored for more efficiency. The question of tackling other measures, e.g., AUC-ROC, also remains open.

\section{Acknowledgments}
Anirban Dasgupta would like to acknowledge the support received from Google, DST(SERB) and Cisco as well as the N Rama Rao Chair position at IIT Gandhinagar. Jayesh Malaviya and Rachit Chhaya would like to acknowledge the support from IIT- Gandhinagar and Dhirubhai Ambani Institute of Information and Communication Technology (DA-IICT) , Gandhinagar, India.
%\newpage
%\pagebreak

%\bibliographystyle{plainnat}
\bibliography{aaai24}

\begin{thebibliography}{28}
\providecommand{\natexlab}[1]{#1}

\bibitem[{Agarwal et~al.(2005)Agarwal, Har-Peled, Varadarajan
  et~al.}]{agarwal2005geometric}
Agarwal, P.~K.; Har-Peled, S.; Varadarajan, K.~R.; et~al. 2005.
\newblock Geometric approximation via coresets.
\newblock \emph{Combinatorial and computational geometry}, 52(1): 1--30.

\bibitem[{Bachem, Lucic, and Krause(2017)}]{bachem2017practical}
Bachem, O.; Lucic, M.; and Krause, A. 2017.
\newblock Practical coreset constructions for machine learning.
\newblock \emph{arXiv preprint arXiv:1703.06476}.

\bibitem[{Bachem, Lucic, and Krause(2018)}]{bachem2018scalable}
Bachem, O.; Lucic, M.; and Krause, A. 2018.
\newblock Scalable k-means clustering via lightweight coresets.
\newblock In \emph{Proceedings of the 24th ACM SIGKDD International Conference
  on Knowledge Discovery \& Data Mining}, 1119--1127.

\bibitem[{Becker and Kohavi(1996)}]{misc_adult_2}
Becker, B.; and Kohavi, R. 1996.
\newblock {Adult}.
\newblock UCI Machine Learning Repository.
\newblock {DOI}: https://doi.org/10.24432/C5XW20.

\bibitem[{B{\'e}n{\'e}dict et~al.(2022)B{\'e}n{\'e}dict, Koops, Odijk, and
  de~Rijke}]{benedict2022sigmoidf1}
B{\'e}n{\'e}dict, G.; Koops, H.~V.; Odijk, D.; and de~Rijke, M. 2022.
\newblock sigmoidF1: A Smooth F1 Score Surrogate Loss for Multilabel
  Classification.
\newblock \emph{Transactions on Machine Learning Research}.

\bibitem[{Blackard(1998)}]{misc_covertype_31}
Blackard, J. 1998.
\newblock {Covertype}.
\newblock UCI Machine Learning Repository.
\newblock {DOI}: https://doi.org/10.24432/C50K5N.

\bibitem[{Braverman et~al.(2022)Braverman, Cohen-Addad, Jiang, Krauthgamer,
  Schwiegelshohn, Toftrup, and Wu}]{braverman2022power}
Braverman, V.; Cohen-Addad, V.; Jiang, H.-C.~S.; Krauthgamer, R.;
  Schwiegelshohn, C.; Toftrup, M.~B.; and Wu, X. 2022.
\newblock The power of uniform sampling for coresets.
\newblock In \emph{2022 IEEE 63rd Annual Symposium on Foundations of Computer
  Science (FOCS)}, 462--473. IEEE.

\bibitem[{Braverman et~al.(2016)Braverman, Feldman, Lang, Statman, and
  Zhou}]{braverman2016new}
Braverman, V.; Feldman, D.; Lang, H.; Statman, A.; and Zhou, S. 2016.
\newblock New frameworks for offline and streaming coreset constructions.
\newblock \emph{arXiv preprint arXiv:1612.00889}.

\bibitem[{Cohen and Peng(2015)}]{cohen2015lp}
Cohen, M.~B.; and Peng, R. 2015.
\newblock Lp row sampling by lewis weights.
\newblock In \emph{Proceedings of the forty-seventh annual ACM symposium on
  Theory of computing}, 183--192.

\bibitem[{Drineas, Mahoney, and Muthukrishnan(2006)}]{drineas2006sampling}
Drineas, P.; Mahoney, M.~W.; and Muthukrishnan, S. 2006.
\newblock Sampling algorithms for l 2 regression and applications.
\newblock In \emph{Proceedings of the seventeenth annual ACM-SIAM symposium on
  Discrete algorithm}, 1127--1136.

\bibitem[{Eban et~al.(2017)Eban, Schain, Mackey, Gordon, Rifkin, and
  Elidan}]{Google_scalable_nondecomposable}
Eban, E.; Schain, M.; Mackey, A.; Gordon, A.; Rifkin, R.; and Elidan, G. 2017.
\newblock Scalable learning of non-decomposable objectives.
\newblock In \emph{Artificial intelligence and statistics}, 832--840. PMLR.

\bibitem[{Feldman(2020)}]{feldman2020core}
Feldman, D. 2020.
\newblock Core-sets: Updated survey.
\newblock \emph{Sampling Techniques for Supervised or Unsupervised Tasks},
  23--44.

\bibitem[{Feldman and Langberg(2011)}]{feldman2011unified}
Feldman, D.; and Langberg, M. 2011.
\newblock A unified framework for approximating and clustering data.
\newblock In \emph{Proceedings of the forty-third annual ACM symposium on
  Theory of computing}, 569--578.

\bibitem[{Joachims(2005)}]{joachims2005support}
Joachims, T. 2005.
\newblock A support vector method for multivariate performance measures.
\newblock In \emph{Proceedings of the 22nd international conference on Machine
  learning}, 377--384.

\bibitem[{Kar, Narasimhan, and Jain(2014)}]{online_sgd_nondecomposable}
Kar, P.; Narasimhan, H.; and Jain, P. 2014.
\newblock Online and stochastic gradient methods for non-decomposable loss
  functions.
\newblock \emph{Advances in Neural Information Processing Systems}, 27.

\bibitem[{Langberg and Schulman(2010)}]{langberg2010universal}
Langberg, M.; and Schulman, L.~J. 2010.
\newblock Universal $\varepsilon$-approximators for integrals.
\newblock In \emph{Proceedings of the twenty-first annual ACM-SIAM symposium on
  Discrete Algorithms}, 598--607. SIAM.

\bibitem[{Li, Long, and Srinivasan(2001)}]{li2001improved}
Li, Y.; Long, P.~M.; and Srinivasan, A. 2001.
\newblock Improved bounds on the sample complexity of learning.
\newblock \emph{Journal of Computer and System Sciences}, 62(3): 516--527.

\bibitem[{Lu, Raff, and Holt(2023)}]{lu2023coreset}
Lu, F.; Raff, E.; and Holt, J. 2023.
\newblock A Coreset Learning Reality Check.
\newblock \emph{arXiv preprint arXiv:2301.06163}.

\bibitem[{Mai, Musco, and Rao(2021)}]{classification_coreset}
Mai, T.; Musco, C.; and Rao, A. 2021.
\newblock Coresets for classification--simplified and strengthened.
\newblock \emph{Advances in Neural Information Processing Systems}, 34:
  11643--11654.

\bibitem[{Munteanu et~al.(2018)Munteanu, Schwiegelshohn, Sohler, and
  Woodruff}]{on_coreset_logistic}
Munteanu, A.; Schwiegelshohn, C.; Sohler, C.; and Woodruff, D. 2018.
\newblock On coresets for logistic regression.
\newblock \emph{Advances in Neural Information Processing Systems}, 31.

\bibitem[{Nan et~al.(2012)Nan, Chai, Lee, and Chieu}]{nan2012optimizing}
Nan, Y.; Chai, K.~M.; Lee, W.~S.; and Chieu, H.~L. 2012.
\newblock Optimizing F-measure: A tale of two approaches.
\newblock \emph{arXiv preprint arXiv:1206.4625}.

\bibitem[{Narasimhan, Kar, and Jain(2015)}]{narasimhan2015optimizing}
Narasimhan, H.; Kar, P.; and Jain, P. 2015.
\newblock Optimizing non-decomposable performance measures: A tale of two
  classes.
\newblock In \emph{International Conference on Machine Learning}, 199--208.
  PMLR.

\bibitem[{Poms et~al.(2021)Poms, Sarukkai, Mullapudi, Sohoni, Mark, Ramanan,
  and Fatahalian}]{poms2021low}
Poms, F.; Sarukkai, V.; Mullapudi, R.~T.; Sohoni, N.~S.; Mark, W.~R.; Ramanan,
  D.; and Fatahalian, K. 2021.
\newblock Low-shot validation: Active importance sampling for estimating
  classifier performance on rare categories.
\newblock In \emph{Proceedings of the IEEE/CVF International Conference on
  Computer Vision}, 10705--10714.

\bibitem[{Samadian et~al.(2020)Samadian, Pruhs, Moseley, Im, and
  Curtin}]{samadian2020unconditional}
Samadian, A.; Pruhs, K.; Moseley, B.; Im, S.; and Curtin, R. 2020.
\newblock Unconditional coresets for regularized loss minimization.
\newblock In \emph{International Conference on Artificial Intelligence and
  Statistics}, 482--492. PMLR.

\bibitem[{Sanyal et~al.(2018)Sanyal, Kumar, Kar, Chawla, and
  Sebastiani}]{opt_nondecompp_deep_network}
Sanyal, A.; Kumar, P.; Kar, P.; Chawla, S.; and Sebastiani, F. 2018.
\newblock Optimizing non-decomposable measures with deep networks.
\newblock \emph{Machine Learning}, 107: 1597--1620.

\bibitem[{Sawade, Landwehr, and Scheffer(2010)}]{sawade2010active}
Sawade, C.; Landwehr, N.; and Scheffer, T. 2010.
\newblock Active estimation of f-measures.
\newblock \emph{Advances in Neural Information Processing Systems}, 23.

\bibitem[{Stolfo et~al.(1999)Stolfo, Fan, Lee, Prodromidis, and
  Chan}]{misc_kdd_cup_1999_data_130}
Stolfo, S.; Fan, W.; Lee, W.; Prodromidis, A.; and Chan, P. 1999.
\newblock {KDD Cup 1999 Data}.
\newblock UCI Machine Learning Repository.
\newblock {DOI}: https://doi.org/10.24432/C51C7N.

\bibitem[{Tukan et~al.(2021)Tukan, Baykal, Feldman, and Rus}]{on_coreset_svm}
Tukan, M.; Baykal, C.; Feldman, D.; and Rus, D. 2021.
\newblock On coresets for support vector machines.
\newblock \emph{Theoretical Computer Science}, 890: 171--191.

\end{thebibliography}

\newpage

\appendix

%% MCC Proof starts from here...

\section{Weak coreset for $MCC$}

\begin{theorem}
Let $\epsilon > 0$. Consider an instance where number of positive samples are $Y^+$ and number of negative samples are $Y^-$, and $n = Y^+ + Y^-$. Let $T$ to be the ground truth positive $0/1$ labels and $P$ to be the predicted positive $0/1$ labels. Let $tp, fp, fn$ be the true positive, false positive and false negative on the full data, $T^{'} =  \frac{\sum_{i} T_i}{n} = \frac{tp + fn}{n} = \frac{\left | Y^+  \right |}{n}$ and $P^{'} =  \frac{\sum_{i} P_i}{n} = \frac{tp + fp}{n}$. We consider $Q_{\gamma}$ to be the set of queries such that $tp \ge \gamma \cdot n$ and $tn \ge \gamma \cdot n$  for $q\in Q_{\gamma}$. Let $d=vc-dimension\left( Q_\gamma \right)$. We claim that uniform sampling with $\left(\frac{1}{\epsilon^2} \left ( d  + \log \frac{1}{\delta} \right) \left( 2 + 2 \cdot \left (\frac{Y^-}{Y^+} \right)^{2}\right) \right)$ samples would be able to give a coreset for $Q_{\gamma}$ that satisfies $\frac{MCC(q)}{\left  ( 1 + \frac{\epsilon}{\gamma}\right)} - 2 \cdot \epsilon \cdot C \le \widetilde{MCC}(q) \le \frac{MCC(q)}{\left  ( 1 - \frac{\epsilon}{\gamma}\right)} + 2 \cdot \epsilon \cdot C^{'}$ for all query $q \in Q_{\gamma}$ with probability $1-4\delta$, where $C \le \frac{ \left( \frac{1}{\gamma} \right)}{ \left (1 + \frac{\epsilon}{\gamma} \right) \sqrt{T^{'} (1- T^{'}) \cdot \gamma}}$ and $C^{'} \le \frac{ \left( \frac{1}{\gamma} \right)}{ \left (1 - \frac{\epsilon}{\gamma} \right) \sqrt{T^{'} (1- T^{'}) \cdot \gamma}}$.

\end{theorem}

\begin{proof}

We consider MCC(matthews correlation coefficient) is defined in the follwing form,

\begin{align*}
MCC & = \frac{\frac{tp}{n} - T^{'}P^{'}}{\sqrt{T^{'} P^{'} (1- T^{'}) (1- P^{'})}} \\
& = \frac{\frac{tp}{n} - T^{'}\left( \frac{tp + fp}{n}\right)}{\sqrt{T^{'} (1- T^{'}) \left( \frac{tp + fp}{n}\right) \left( 1 - \left( \frac{tp + fp}{n}\right)\right)}}\\
& = \frac{tp - T^{'}\left( tp + fp\right)}{\sqrt{T^{'} (1- T^{'}) \left( tp + fp \right) \left( tn + fn\right)}}
\end{align*}

Now, we would require approximation for all four quantity viz. $tp$, $tn$, $fp$ and $fn$ for all query, and for that we would use Theorem 5.1,

%%%  tp approximation and sample size ...
For $tp$ approximation,
$$\forall q \in Q; d_v \left(\sum_{x \in Y^+}\frac{1}{Y^+} \delta_x, \frac{1}{|S_1|} \sum_{x \in S_1}\delta_x\right ) \le \alpha_1$$

where, $\delta_x = 1(x \in tp)$ and $d_v(a, b) = \frac{|a - b|}{a+b+v}$

\begin{align*}
     &= \frac{\left | \sum_{x \in Y^+}\frac{1}{Y^+} \delta_x - \frac{1}{|S_1|} \sum_{x \in S_1}\delta_x \right |}{ \sum_{x \in Y^+}\frac{1}{Y^+} \delta_x + \frac{1}{|S_1|} \sum_{x \in S_1}\delta_x + v} \le \alpha_1\\
     & = \frac{\left | \frac{tp}{Y^+} - \frac{\widetilde{tp}}{|S_1|} \right |}{ \frac{tp}{Y^+} + \frac{\widetilde{tp}}{|S_1|} + v} \le \alpha_1 \\
     & = \left | \frac{tp}{Y^+} - \frac{\widetilde{tp}}{|S_1|} \right | \le 3 \alpha_1
\end{align*}

Since, $\frac{tp}{Y^+}$ and $\frac{\widetilde{tp}}{|S_1|}$ is less than one, lets take $v = \frac{1}{2}$.

$$\widetilde{tp} \in tp \left (\frac{S_1}{Y^+}\right) \pm S_1 \cdot (3\alpha_1)$$ \\

where, $\epsilon = (3\alpha_1)$ for above additive error epsilon, and therefore, $\alpha_1 = \frac{\epsilon}{3}$ \\

Now, size of samples require to satisfy the above approximation using Theorem 5.1 is,
$S_1 = \Omega\left( \frac{1}{\alpha_1^2 \cdot v} \left ( d \log \frac{1}{v} + \log \frac{1}{\delta} \right)\right) = \Omega\left( \frac{1}{\epsilon^2} \left ( d  + \log \frac{1}{\delta} \right)\right) $, with probability $1 - \delta$ for all $q \in Q$ where $Q_\gamma \subset Q$ . \\

%%% end of tp calculation...

%%% fn approximation and sample size ...

For $fn$ approximation, 

$$\forall q \in Q; d_v \left(\sum_{y \in Y^+}\frac{1}{Y^+} \delta_y, \frac{1}{|S_2|} \sum_{y \in S_2}\delta_y\right ) \le \alpha_2$$

where, $\delta_y = 1(y \in fn)$ and $d_v(a, b) = \frac{|a - b|}{a+b+v}$

\begin{align*}
     &= \frac{\left | \sum_{y \in Y^+}\frac{1}{Y^+} \delta_y - \frac{1}{|S_2|} \sum_{y \in S_2}\delta_y \right |}{ \sum_{y \in Y^+}\frac{1}{Y^+} \delta_y + \frac{1}{|S_2|} \sum_{y \in S_2}\delta_y + v} \le \alpha_2\\
     & = \frac{\left | \frac{fn}{Y^+} - \frac{\widetilde{fn}}{|S_2|} \right |}{ \frac{fn}{Y^+} + \frac{\widetilde{fn}}{|S_2|} + v} \le \alpha_2 \\
     & = \left | \frac{fn}{Y^+} - \frac{\widetilde{fn}}{|S_2|} \right | \le 3 \alpha_2
\end{align*}

Since, $\frac{fn}{Y^+}$ and $\frac{\widetilde{fn}}{|S_2|}$ is less than one, lets take $v = \frac{1}{2}$.

$$\widetilde{fn} \in fn \left (\frac{S_2}{Y^+}\right) \pm S_2 \cdot (3\alpha_2)$$ \\

where, $\epsilon = (3\alpha_2)$ for above additive error epsilon, and therefore, $\alpha_2 = \frac{\epsilon}{3}$ \\

Now, size of samples require to satisfy the above approximation using Theorem 5.1 is,
$S_2 = \Omega\left( \frac{1}{\alpha_2^2 \cdot v} \left ( d \log \frac{1}{v} + \log \frac{1}{\delta} \right)\right) = \Omega\left( \frac{1}{\epsilon^2} \left ( d  + \log \frac{1}{\delta} \right)\right) $, with probability $1 - \delta$ for all $q \in Q$ where $Q_\gamma \subset Q$ . \\

%%% end of fn calculation...

%%% fp approximation and sample size ...

For $fp$ approximation, 

$$\forall q \in Q; d_v \left(\sum_{z \in Y^-}\frac{1}{Y^-}  \delta_z, \frac{1}{|S_3|} \sum_{z \in S_3}\delta_z \right ) \le \alpha_3$$

where, $\delta_z = 1(z \in fp)$ and $d_v(a, b) = \frac{|a - b|}{a+b+v}$

\begin{align*}
     &= \frac{\left | \sum_{z \in Y^-}\frac{1}{Y^-} \delta_z - \frac{1}{|S_3|} \sum_{z \in S_3}\delta_z \right |}{ \sum_{z \in Y^-}\frac{1}{Y^-} \delta_z + \frac{1}{|S_3|} \sum_{z \in S_3}\delta_z + v} \le \alpha_3\\
     & = \frac{\left | \frac{fp}{Y^-} - \frac{\widetilde{fp}}{|S_3|} \right |}{ \frac{fp}{Y^-} + \frac{\widetilde{fp}}{|S_3|} + v} \le \alpha_3 \\
     & = \left | \frac{fp}{Y^-} - \frac{\widetilde{fp}}{|S_3|} \right | \le 3 \alpha_3
\end{align*}

Since, $\frac{fp}{Y^-}$ and $\frac{\widetilde{fp}}{|S_3|}$ is less than one, lets take $v = \frac{1}{2}$.

$$\widetilde{fp} \in fp \left (\frac{S_3}{Y^-}\right) \pm S_3 \cdot (3\alpha_3)$$

%%%%%

where, $\epsilon = (3\alpha_3) $ for above additive error epsilon, and therefore, $\alpha_3 = \frac{\epsilon}{3}$ \\

Let, $\widetilde{\widetilde{fp}}$ defined as following,

$$\widetilde{\widetilde{fp}} = \left (\frac{Y^-}{Y^+} \right) \widetilde{fp} \in fp \left (\frac{S_3}{Y^+}\right) \pm S_3 \cdot (3\alpha_3) \left (\frac{Y^-}{Y^+} \right) $$ \\

Basically we are reweighing all negative class data points by $ \left (\frac{Y^+}{Y^-} \right)$ and thus this term will also comes in the calculation of the size to satisfy $\widetilde{\widetilde{fp}}$. \\

After applying reweighing our modified $\alpha_3$ is defined as, $\alpha_{3}^{'} = \frac{\epsilon}{3} \left (\frac{Y^+}{Y^-} \right)$. 

Now, size of samples require to satisfy the above approximation using Theorem 5.1 is,
$S_3 = \Omega\left( \frac{1}{(\alpha_{3}^{'})^2 \cdot v} \left ( d \log \frac{1}{v} + \log \frac{1}{\delta} \right)\right) = \Omega\left( \left (\frac{Y^-}{Y^+} \right)^{2} \frac{1}{\epsilon^2} \left ( d  + \log \frac{1}{\delta} \right)\right) $, with probability $1 - \delta$ for all $q \in Q$ where $Q_\gamma \subset Q$ . \\

%%% end of fp calculation...

%%% tn approximation and sample size ...

For $tn$ approximation, 

$$\forall q \in Q; d_v \left(\sum_{w \in Y^-}\frac{1}{Y^-} \delta_w, \frac{1}{|S_4|} \sum_{w \in S_4}\delta_w\right ) \le \alpha_4$$

where, $\delta_w = 1(w \in tn)$ and $d_v(a, b) = \frac{|a - b|}{a+b+v}$

\begin{align*}
     &= \frac{\left | \sum_{w \in Y^-}\frac{1}{Y^-} \delta_w - \frac{1}{|S_4|} \sum_{w \in S_4}\delta_w \right |}{ \sum_{w \in Y^-}\frac{1}{Y^-} \delta_w + \frac{1}{|S_4|} \sum_{w \in S_4}\delta_w + v} \le \alpha_4\\
     & = \frac{\left | \frac{tn}{Y^-} - \frac{\widetilde{tn}}{|S_4|} \right |}{ \frac{tn}{Y^-} + \frac{\widetilde{tn}}{|S_4|} + v} \le \alpha_4 \\
     & = \left | \frac{tn}{Y^-} - \frac{\widetilde{tn}}{|S_4|} \right | \le 3 \alpha_4
\end{align*}

Since, $\frac{tn}{Y^-}$ and $\frac{\widetilde{tn}}{|S_4|}$ is less than one, lets take $v = \frac{1}{2}$.

$$\widetilde{tn} \in tn \left (\frac{S_4}{Y^-}\right) \pm S_4 \cdot (3\alpha_4)$$

%%%%%%%

where, $\epsilon = (3\alpha_4) $ for above additive error epsilon, and therefore, $\alpha_4 = \frac{\epsilon}{3}$ \\

Let, $\widetilde{\widetilde{tn}}$ defined as following,

$$\widetilde{\widetilde{tn}} = \left (\frac{Y^-}{Y^+} \right) \widetilde{tn} \in tn \left (\frac{S_4}{Y^+}\right) \pm S_4 \cdot (3\alpha_4) \left (\frac{Y^-}{Y^+} \right) $$ \\

Basically we are reweighing all negative class data points by $ \left (\frac{Y^+}{Y^-} \right)$ and thus this term will also comes in the calculation of the size to satisfy $\widetilde{\widetilde{tn}}$. \\

After applying reweighing our modified $\alpha_4$ is defined as, $\alpha_{4}^{'} = \frac{\epsilon}{3} \left (\frac{Y^+}{Y^-} \right)$. 

Now, size of samples require to satisfy the above approximation using Theorem 5.1 is,
$S_4 = \Omega\left( \frac{1}{(\alpha_{4}^{'})^2 \cdot v} \left ( d \log \frac{1}{v} + \log \frac{1}{\delta} \right)\right) = \Omega\left( \left (\frac{Y^-}{Y^+} \right)^{2} \frac{1}{\epsilon^2} \left ( d  + \log \frac{1}{\delta} \right)\right) $, with probability $1 - \delta$ for all $q \in Q$ where $Q_\gamma \subset Q$ . \\

%%%%%%%

%%% end of tn calculation...

Thus, in-order to satisfy all four approximations viz. $tp$, $tn$, $fn$ and $fp$ we would require total sample size, 

\begin{align*}
S &= S_1 + S_2 + S_3 + S_4\\
& = \left(\frac{1}{\epsilon^2} \left ( d  + \log \frac{1}{\delta} \right) \left( 2 + 2 \cdot \left (\frac{Y^-}{Y^+} \right)^{2}\right) \right)
\end{align*}

Thus, with probability $ 1- 4\delta$, above approximation for $tp$, $tn$, $fn$ and $fp$ holds. \\

Now, for the MCC coreset left hand side guarantee,
\newpage
\onecolumn

\begin{align*}
    \begin{split}
        &\widetilde{MCC} =  \frac{\tilde{tp} - T^{'}\left( \tilde{tp} + \tilde{\tilde{fp}}\right)}{\sqrt{T^{'} (1- T^{'}) \left( \tilde{tp} + \tilde{\tilde{fp}} \right) \left( \tilde{\tilde{tn}} + \tilde{fn}\right)}} \\
        &\ge \frac{\frac{tp \cdot S}{Y^+} - ( S \cdot \epsilon) - T^{'} \left ( \frac{tp \cdot S}{Y^+} + ( S \cdot \epsilon) + \frac{fp \cdot S}{Y^+} + ( S \cdot \epsilon) \frac{Y^-}{Y^+} \right)}{\sqrt{T^{'} (1- T^{'}) \left[ \frac{tp \cdot S}{Y^+} + ( S \cdot \epsilon) + \frac{fp \cdot S}{Y^+} + ( S \cdot \epsilon) \frac{Y^-}{Y^+}  \right] \left[ \frac{tn \cdot S}{Y^+} + ( S \cdot \epsilon) \frac{Y^-}{Y^+} + \frac{fn \cdot S}{Y^+} + ( S \cdot \epsilon) \right]}} \\
        & = \frac{\frac{S}{Y^+} \left [ tp - (\epsilon \cdot Y^+) - T^{'} \left ( tp + fp + ( \epsilon \cdot Y^+) + (\epsilon \cdot Y^-) \right) \right]}{\sqrt{T^{'} (1- T^{'}) \frac{S}{Y^+} \left[ tp + fp + (\epsilon \cdot Y^+) + (\epsilon \cdot Y^-) \right] \frac{S}{Y^+} \left[ tn + fn + (\epsilon \cdot Y^-) + (\epsilon \cdot Y^+) \right]}} \\
        & = \frac{\frac{S}{Y^+} \left [ tp - T^{'}(tp + fp) - \left[ \epsilon \cdot Y^+ + T^{'} \cdot \epsilon \cdot n \right] \right]}{\sqrt{T^{'} (1- T^{'}) \frac{S}{Y^+} \left[ tp + fp + \epsilon \cdot n \right] \frac{S}{Y^+} \left[ tn + fn + \epsilon \cdot n \right]}} \\
        & = \frac{ tp - T^{'}(tp + fp) - 2 \cdot \epsilon \cdot Y^+}{\sqrt{T^{'} (1- T^{'}) \left[ tp + fp + \epsilon \cdot n \right] \left[ tn + fn + \epsilon \cdot n \right]}}\\
        & \ge \frac{ tp - T^{'}(tp + fp) - 2 \cdot \epsilon \cdot n}{\sqrt{T^{'} (1- T^{'}) \left (1 + \frac{\epsilon}{\gamma} \right) \left( tp + fp \right) \left (1 + \frac{\epsilon}{\gamma} \right) \left( tn + fn \right)}} ;  tp, tn \ge \gamma \cdot n\\
        & \ge \frac{ \left( 1  - \frac{2 \cdot \epsilon}{\gamma}\right) tp - T^{'}(tp + fp)}{\sqrt{T^{'} (1- T^{'}) \left (1 + \frac{\epsilon}{\gamma} \right) \left( tp + fp \right) \left (1 + \frac{\epsilon}{\gamma} \right) \left( tn + fn \right)}} \\
        & = \frac{ tp - T^{'}(tp + fp) -  \left( \frac{2 \cdot \epsilon}{\gamma} \right) \cdot tp}{ \left (1 + \frac{\epsilon}{\gamma} \right) \sqrt{T^{'} (1- T^{'}) \left( tp + fp \right) \left( tn + fn \right)}} \\
        & \ge \frac{MCC}{\left (1 + \frac{\epsilon}{\gamma} \right)} -  2 \cdot \epsilon \cdot C      
    \end{split}  
\end{align*}

Where, $C = \frac{ \left( \frac{1}{\gamma} \right) \cdot tp}{ \left (1 + \frac{\epsilon}{\gamma} \right) \sqrt{T^{'} (1- T^{'}) \left( tp + fp \right) \left( tn + fn \right)}}$, , which can be upper bounded by the following,

 $$C \le \frac{ \left( \frac{1}{\gamma} \right)}{ \left (1 + \frac{\epsilon}{\gamma} \right) \sqrt{T^{'} (1- T^{'}) \cdot \gamma}}$$ \\

We can have similar guarantee for the right hand side as follows.

\begin{align*}
    \begin{split}
        &\widetilde{MCC} =  \frac{\tilde{tp} - T^{'}\left( \tilde{tp} + \tilde{\tilde{fp}}\right)}{\sqrt{T^{'} (1- T^{'}) \left( \tilde{tp} + \tilde{\tilde{fp}} \right) \left( \tilde{\tilde{tn}} + \tilde{fn}\right)}} \\
        &\le \frac{\frac{tp \cdot S}{Y^+} + ( S \cdot \epsilon) - T^{'} \left ( \frac{tp \cdot S}{Y^+} - ( S \cdot \epsilon) + \frac{fp \cdot S}{Y^+} - ( S \cdot \epsilon) \frac{Y^-}{Y^+} \right)}{\sqrt{T^{'} (1- T^{'}) \left[ \frac{tp \cdot S}{Y^+} - ( S \cdot \epsilon) + \frac{fp \cdot S}{Y^+} - ( S \cdot \epsilon) \frac{Y^-}{Y^+}  \right] \left[ \frac{tn \cdot S}{Y^+} - ( S \cdot \epsilon) \frac{Y^-}{Y^+} + \frac{fn \cdot S}{Y^+} - ( S \cdot \epsilon) \right]}} \\
        & = \frac{\frac{S}{Y^+} \left [ tp + (\epsilon \cdot Y^+) - T^{'} \left ( tp + fp - ( \epsilon \cdot Y^+) + (\epsilon \cdot Y^-) \right) \right]}{\sqrt{T^{'} (1- T^{'}) \frac{S}{Y^+} \left[ tp + fp - (\epsilon \cdot Y^+) + (\epsilon \cdot Y^-) \right] \frac{S}{Y^+} \left[ tn + fn - (\epsilon \cdot Y^-) + (\epsilon \cdot Y^+) \right]}} \\
        & = \frac{\frac{S}{Y^+} \left [ tp - T^{'}(tp + fp) + \left[ \epsilon \cdot Y^+ + T^{'} \cdot \epsilon \cdot n \right] \right]}{\sqrt{T^{'} (1- T^{'}) \frac{S}{Y^+} \left[ tp + fp - \epsilon \cdot n \right] \frac{S}{Y^+} \left[ tn + fn - \epsilon \cdot n \right]}} \\
        & = \frac{ tp - T^{'}(tp + fp) + 2 \cdot \epsilon \cdot Y^+}{\sqrt{T^{'} (1- T^{'}) \left[ tp + fp - \epsilon \cdot n \right] \left[ tn + fn - \epsilon \cdot n \right]}}\\
        & \le \frac{ tp - T^{'}(tp + fp) + 2 \cdot \epsilon \cdot n}{\sqrt{T^{'} (1- T^{'}) \left (1 - \frac{\epsilon}{\gamma} \right) \left( tp + fp \right) \left (1 - \frac{\epsilon}{\gamma} \right) \left( tn + fn \right)}} ;  tp, tn \ge \gamma \cdot n \\
        & \le \frac{ \left( 1  + \frac{2 \cdot \epsilon}{\gamma}\right) tp - T^{'}(tp + fp)}{\sqrt{T^{'} (1- T^{'}) \left (1 - \frac{\epsilon}{\gamma} \right) \left( tp + fp \right) \left (1 - \frac{\epsilon}{\gamma} \right) \left( tn + fn \right)}} \\
        & = \frac{ tp - T^{'}(tp + fp) +  \left( \frac{2 \cdot \epsilon}{\gamma} \right) \cdot tp}{ \left (1 - \frac{\epsilon}{\gamma} \right) \sqrt{T^{'} (1- T^{'}) \left( tp + fp \right) \left( tn + fn \right)}} \\
        & \le \frac{MCC}{\left (1 - \frac{\epsilon}{\gamma} \right)} +  2 \cdot \epsilon \cdot C^{'}      
    \end{split}  
\end{align*}
Where, $ C^{'} = \frac{ \left( \frac{1}{\gamma} \right) \cdot tp}{ \left (1 - \frac{\epsilon}{\gamma} \right) \sqrt{T^{'} (1- T^{'}) \left( tp + fp \right) \left( tn + fn \right)}} $, which can be upper bounded by the following,

$$C^{'} \le \frac{ \left( \frac{1}{\gamma} \right)}{ \left (1 - \frac{\epsilon}{\gamma} \right) \sqrt{T^{'} (1- T^{'}) \cdot \gamma}}$$

\end{proof}

%% MCC proof end here..

\section{More Experimental Results}

Some more experiments with Adult dataset and MCC loss function are shown in the Figure 5. In that also, we can notice uniform sampling performs better or comparable with some computationally complex methods.

\begin{figure}
     \centering
     \begin{subfigure}[b]{0.3\textwidth}
         \centering
         \includegraphics[width= 1.25\textwidth, height=0.20\textheight]{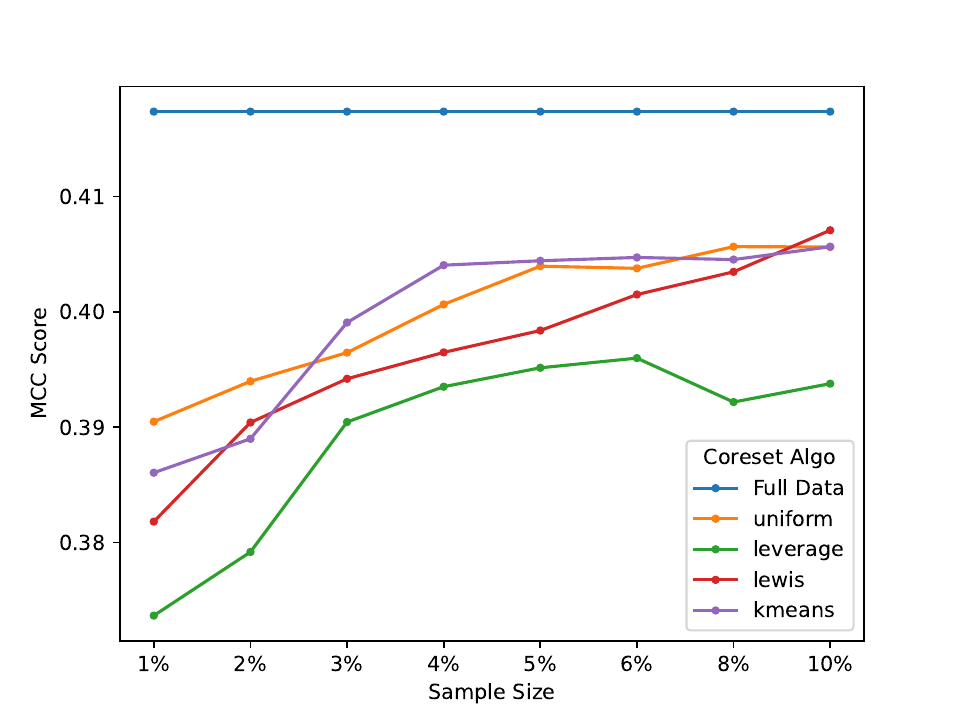}
         \caption{Adult Logistic}
         \label{fig:Adult Logistic}
     \end{subfigure}
     \hfill
     \begin{subfigure}[b]{0.3\textwidth}
         \centering
         \includegraphics[width= 1.25\textwidth, height=0.20\textheight]{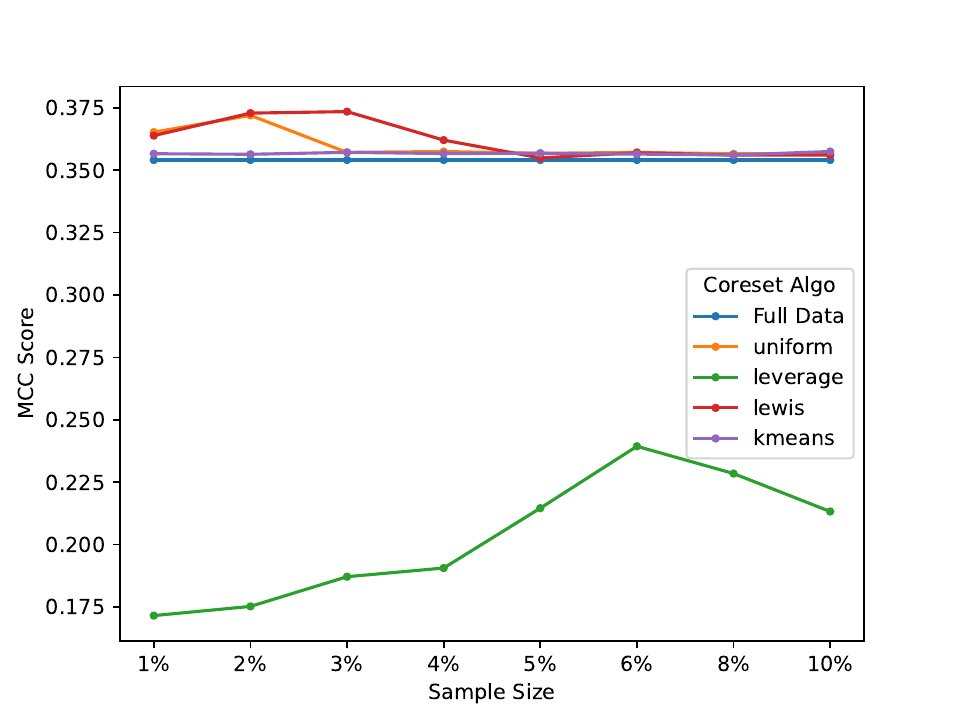}
         \caption{Adult SVM}
         \label{fig:Adult SVM}
     \end{subfigure}
     \hfill
     \begin{subfigure}[b]{0.3\textwidth}
         \centering
         \includegraphics[width= 1.25\textwidth, height=0.20\textheight]{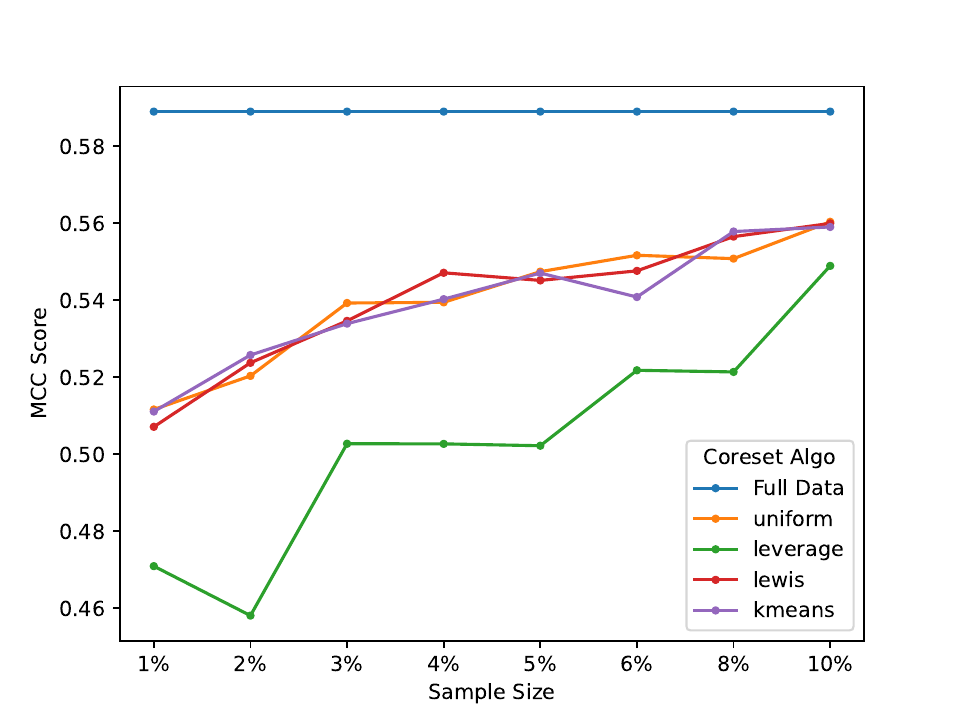}
         \caption{Adult MLP}
         \label{fig:Adult MLP}
     \end{subfigure}
        \caption{$MCC$ Score on Adult dataset for different classifiers and different coreset strategies.}
        \label{fig:MCC Score on Adult}
\end{figure}

%%%%%\end{comment} Appendix ends here.....

%%% rebuttal extra experiments 

The following tables show F1 score and MCC respectively for different coreset sizes and strategies for some other classifiers also. Here LR represents Logistic Regression, GB represents Gradient Boosting, and DT represents Decision Tree classifiers, respectively.

%% Full dataset table start

\begin{table}[!ht]
    \centering
    \begin{tabular}{|c|c|c|}
    \hline
        \textbf{Model Name} & \textbf{F1 Score} & \textbf{MCC Score} \\ \hline
        Logistic Regression & 0.5088 & 0.4394 \\ \hline
        Decision Tree & 0.5487 & 0.4862 \\ \hline
        Gradient Boosting & 0.5360 & 0.4716 \\ \hline
    \end{tabular}
    \caption{F1-Score and MCC-Score on Bank Marketing Full Dataset}
\end{table}

%% Full dataset table end

\begin{table}[!ht]
    \centering
    \begin{tabular}{|c|c|c|c|c|c|}
    \hline
        \textbf{Model} & \textbf{Coreset} & \textbf{2\%} & \textbf{5\%} & \textbf{8\%} & \textbf{10\%} \\ \hline
        LR & uniform & 0.4331 & 0.4335 & 0.4412 & 0.4400 \\ \hline
        LR & leverage & 0.3125 & 0.3356 & 0.3366 & 0.3323 \\ \hline
        LR & lewis & 0.4346 & 0.4260 & 0.4342 & 0.4282 \\ \hline
        LR & kmeans & 0.4002 & 0.3861 & 0.3922 & 0.3842 \\ \hline
        DT & uniform & 0.4623 & 0.5171 & 0.5322 & 0.5290 \\ \hline
        DT & leverage & 0.4482 & 0.4879 & 0.4880 & 0.4898 \\ \hline
        DT & lewis & 0.4889 & 0.5124 & 0.5217 & 0.5216 \\ \hline
        DT & kmeans & 0.4429 & 0.4819 & 0.5126 & 0.5184 \\ \hline
        GB & uniform & 0.5623 & 0.5830 & 0.5932 & 0.5972 \\ \hline
        GB & leverage & 0.5475 & 0.5769 & 0.5865 & 0.5862 \\ \hline
        GB & lewis & 0.5654 & 0.5834 & 0.5931 & 0.5946 \\ \hline
        GB & kmeans & 0.5377 & 0.5700 & 0.5775 & 0.5822 \\ \hline
    \end{tabular}
    \caption{F1-Score on Bank Marketing Dataset}
\end{table}

%%MCC table for bank marketing dataset

\begin{table}[!ht]
    \centering
    \begin{tabular}{|c|c|c|c|c|c|}
    \hline
        \textbf{Model} & \textbf{Coreset} & \textbf{2\%} & \textbf{5\%} & \textbf{8\%} & \textbf{10\%} \\ \hline
        LR & uniform & 0.3512 & 0.3516 & 0.3609 & 0.3593 \\ \hline
        LR & leverage & 0.2237 & 0.2536 & 0.2508 & 0.2460 \\ \hline
        LR & lewis & 0.3531 & 0.3434 & 0.3530 & 0.3463 \\ \hline
        LR & kmeans & 0.3123 & 0.2978 & 0.3032 & 0.2937 \\ \hline
        DT & uniform & 0.3926 & 0.4523 & 0.4686 & 0.4638 \\ \hline
        DT & leverage & 0.3748 & 0.4171 & 0.4191 & 0.4206 \\ \hline
        DT & lewis & 0.4219 & 0.4475 & 0.4569 & 0.4549 \\ \hline
        DT & kmeans & 0.3742 & 0.4146 & 0.4480 & 0.4538 \\ \hline
        GB & uniform & 0.5030 & 0.5293 & 0.5416 & 0.5466 \\ \hline
        GB & leverage & 0.4847 & 0.5196 & 0.5313 & 0.5315 \\ \hline
        GB & lewis & 0.5071 & 0.5297 & 0.5413 & 0.5433 \\ \hline
        GB & kmeans & 0.4734 & 0.5119 & 0.5232 & 0.5273 \\ \hline
    \end{tabular}
    \caption{MCC-Score on Bank Marketing Dataset}
\end{table}

\end{document}